\documentclass[final,12pt]{colt2024} %

\usepackage{mathtools}
\usepackage{bbm}
\usepackage{enumitem}

\renewcommand{\epsilon}{\varepsilon}

\usepackage{thmtools,thm-restate}

\usepackage{pgfplots}
\pgfplotsset{compat=1.15}
\usepackage{mathrsfs}
\usetikzlibrary{arrows}
\usetikzlibrary{shapes}

\SetKwInput{KwInput}{Input}
\SetKwInput{KwOutput}{Output}
\SetKw{Continue}{continue}

\title{Efficient Algorithms for Learning Monophonic Halfspaces in Graphs}
\usepackage{times}

\coltauthor{%
 \Name{Marco Bressan} \Email{marco.bressan@unimi.it}\\
 \addr Università degli Studi di Milano, Milan, Italy
 \AND
 \Name{Emmanuel Esposito} \Email{emmanuel@emmanuelesposito.it}\\
 \addr Università degli Studi di Milano, Milan, Italy\\
 \addr Istituto Italiano di Tecnologia, Genoa, Italy
 \AND
 \Name{Maximilian Thiessen} \Email{maximilian.thiessen@tuwien.ac.at}\\
 \addr TU Wien, Vienna, Austria%
}

\newcommand{\scC}{\mathcal{C}} %
\newcommand{\scH}{\mathcal{H}} %
\newcommand{\scHm}{\mathcal{H}_{\textsc{mp}}} %
\newcommand{\scO}{\mathcal{O}}

\newcommand{\Cut}{\delta}
\newcommand{\Border}{\Gamma}
\newcommand{\MonInt}{I}

\newcommand{\algoname}[1]{\ensuremath{\textsc{#1}}\xspace}
\newcommand{\AlgoList}{\algoname{mh-fpt-list}}

\newcommand{\PolyDelay}{\algoname{mh-list}}
\newcommand{\PolyChecker}{\algoname{mh-check}}

\newcommand{\R}{\mathbb{R}}
\newcommand{\omgtil}{\tilde{\omega}}
\newcommand{\CC}{\mathrm{cc}} %
\newcommand{\satvar}{\varphi}

\DeclareMathOperator{\VC}{VC}
\DeclareMathOperator{\VS}{VS}

\DeclareMathOperator{\diam}{diam}
\DeclareMathOperator{\poly}{poly}
\DeclareMathOperator{\qc}{qc} %
\DeclareMathOperator{\conv}{conv}
\DeclareMathOperator*{\argmin}{arg\,min}

\usepackage[framemethod=TikZ]{mdframed}
\newcounter{emmanuelbox}

\begin{document}

\maketitle

\begin{abstract}%
We study the problem of learning a binary classifier on the vertices of a graph. In particular, we consider classifiers given by \emph{monophonic halfspaces}, partitions of the vertices that are convex in a certain abstract sense.
Monophonic halfspaces, and related notions such as geodesic halfspaces, have recently attracted interest, and several connections have been drawn between their properties (e.g., their VC dimension) and the structure of the underlying graph $G$.
We prove several novel results for learning monophonic halfspaces in the supervised, online, and active settings. Our main result is that a monophonic halfspace can be learned with near-optimal passive sample complexity in time polynomial in $n=|V(G)|$. This requires us to devise a polynomial-time algorithm for consistent hypothesis checking, based on several structural insights on monophonic halfspaces and on a reduction to $2$-satisfiability. We prove similar results for the online and active settings.
We also show that the concept class can be enumerated with delay $\poly(n)$, and that empirical risk minimization can be performed in time $2^{\omega(G)}\poly(n)$ where $\omega(G)$ is the clique number of $G$. These results answer open questions from the literature~\citep{gonzalez2020covering}, and show a contrast with geodesic halfspaces, for which some of the said problems are NP-hard~\citep{seiffarth2023maximal}.
\end{abstract}


\section{Introduction}
We study the problem of binary classification of the vertices of a graph. With the advent of social networks and the use of graph-based techniques in machine learning, this problem has received considerable attention in all the most common learning settings, including supervised learning~\citep{hanneke2006analysis,pelckmans2007margin}, active learning~\citep{afshani2007complexity,guillory2009label,cesa2010active,dasarathy2015s2}, and online learning \citep{herbster2005online,cesa2013random,herbster2015online}.
Most results obtained so far rely on the homophily principle, that is, they assume that adjacent vertices tend to belong to the same class. We take a different perspective on this problem and assume that the ground truth classes are \emph{monophonic halfspaces}, a notion related to linear separability and convexity in Euclidean spaces. In this way, we hope to exploit the intuition and technical machinery behind classes of convex concepts, which are often at the heart of machine learning models (e.g., think of intervals, halfspaces, or polytopes).

Let us spend a few words to introduce the notion of graph convexity.
In $\R^d$, a set $C$ is convex if for all $x,y \in C$ the connecting segment $I(x,y)$ lies in $C$. One can port this intuition to graphs in the following way. Given two vertices $x$ and $y$ in a graph $G$, one defines the \emph{interval} $I(x,y)$. For instance, $I(x,y)$ can be the set of all nodes $z \in V(G)$ that lie on some shortest path between $x$ and $y$, that is, such that $d(x,y)=d(x,z)+d(z,y)$ where $d\colon V\times V\to \mathbb{N}$ is the shortest-path distance in $G$. In this way $I(x,y)$ becomes the graph analogue of a segment in $\R^d$. This yields the well-known notion of \emph{geodesic} convexity: a set $C \subseteq V$ is geodesically convex if the induced subgraph $G[C]$ is connected and $I(x,y) \subseteq C$ for all $x,y \in C$, and is a geodesic halfspace if $V \setminus C$ is geodesically convex too. Thus, $C$ is the graph analogue of a halfspace in~$\R^d$. Interestingly, several bounds that hold for halfspaces in $\R^d$ continue to hold: for example, the VC dimension of geodesic halfspaces is bounded by their Radon number. One can therefore port several learning-theoretical results to the realm of graphs.

Starting from decade-old classic results, in recent years there has been a revived interest in the study of graph halfspaces and their learning properties~\citep{duchet1983ensemble,chepoi1986some,chepoi1994separation,farber1986convexity,pelayo2013geodesic,thiessen2021active,bressan2021exact,chalopin2022unlabeled,chalopin2022first,seiffarth2023maximal}.
Most works, however, give purely existential bounds on the invariants (the Radon number, the Helly number, the VC dimension, etc.) and study the connection with properties of the graph (maximum degree, clique number, etc.). %
Little is known, instead, about the existence of efficient algorithms for learning a halfspace in the most common supervised learning settings. Note that the existence of efficient algorithms is not obvious: already for the class of geodesic halfspaces described above, even just deciding if there exists a hypothesis consistent with a given labeled sample of vertices is NP-hard \citep{seiffarth2023maximal}.
In this work we study efficient learning algorithms for another variant of abstract graph halfspaces, called \emph{monophonic halfspaces}, which are defined through what is called minimal convexity or induced-path convexity \citep{farber1986convexity,duchet1988convex}. 

Let us introduce some notation. Given two distinct vertices $x,y \in V$, the \emph{monophonic interval} $\MonInt(x,y)$ between $x$ and $y$ is the set of all vertices that lie on some induced path between $x$ and $y$. A set $C \subseteq V$ is \emph{monophonically convex} (m-convex) if $G[C]$ is connected and $\MonInt(x,y)\subseteq C$ for all $x,y\in C$. A set $H \subseteq V$ is a \emph{monophonic halfspace} (m-halfspace) if both $H$ and $\overline{H} = V \setminus H$ are m-convex. For instance, if $G$ is a tree, then the connected components left by deleting an edge are m-halfspaces; if $G$ is a clique, then any subset is a m-halfspace, see Figure~\ref{fig:halfspace} for another example. In real-world networks, communities and clusters often tend to be geodesically convex, e.g., in gene similarity networks \citep{zhou2002transitive}, protein-protein interaction networks \citep{li2013identification}, community detection benchmark datasets \citep{thiessen2021active}, and collaboration networks \citep{vsubelj2019convexity}. In the latter and many other cases, the set of  monophonic and geodesic convexity actually coincides \citep{malvestuto2012characteristic}.

Monophonic halfspaces are among the most studied graph halfspaces \citep{bandelt1989graphs,changat2005convexities,dourado2010complexity}, second only to geodesic halfspaces. Nonetheless, only a couple of facts are known; for instance, a connection between the Radon number of m-convex sets and the clique number $\omega(G)$, or that computing a $k$-partition of $V$ into m-convex subsets is hard for $k \ge 3$ \citep{gonzalez2020covering}. The present work fills this gap by providing several concrete results.

\subsection*{Our Contributions}

For a graph $G=(V,E)$, let $n=|V|$ and $m=|E|$, and let $\scHm=\scHm(G)$ denote the concept class consisting of all m-halfspaces of $G$.

\begin{enumerate}[leftmargin=.19in,itemsep=0pt,topsep=0pt]
\item In the realizable PAC setting, we give an algorithm that with probability $1-\delta$ learns a monophonic halfspace with accuracy $\epsilon$ by using $\scO\!\left(\epsilon^{-1}\left(\omega(G)\log(1/\varepsilon)+\log(1/\delta)\right)\right)$ labeled samples and $\poly(n)$ time. The heart of this result is a polynomial-time consistency checker, i.e., an algorithm that in time $\poly(n)$ finds a m-halfspace $H \in \scHm$ consistent with a given labeled subset of vertices, if one exists. This polytime consistency checker is nontrivial; it requires us to exploit heavily the structure of monophonic halfspaces, and to devise a polynomial-time reduction to 2-satisfiability by carefully constructing boolean formulas that express constraints on $G$.

\item In the agnostic PAC setting, we give an algorithm that with probability $1-\delta$ learns a monophonic halfspace with accuracy $\epsilon$ by using $\scO\!\left(\epsilon^{-2}\left(\omega(G)+\log(1/\delta)\right)\right)$ samples and $|\scHm|\poly(n) \le 2^{\omega(G)} \poly(n)$ time. To this end we list $\scHm$ in time $|\scHm|\poly(n)$ using the polynomial-time checker above; then, we prove that $|\scHm| \le \frac{4m \, 2^{\omega(G)}}{\omega(G)}+2$ by a careful analysis of a second listing algorithm tailored to the purpose.

\item In the realizable active setting, we give an algorithm that learns $H \in \scHm(G)$ in time $\poly(n)$ using $\scO(\omega(G)+\log\diam(G)+h(G))$ queries, where $\diam(G)$ and $h(G)$ are respectively the diameter and the ``monophonic hull number'' of~$G$. We leave open the problem of obtaining a similar result for agnostic active learning.

\item In the realizable online setting, we give an algorithm that learns $H \in \scHm(G)$ in time $\poly(n)$ by making $\scO(\omega(G)\log n)$ mistakes. We obtain this algorithm by showing a decomposition of m-halfspaces into a small disjunction of simpler concepts, and applying the Winnow algorithm. We give a similar result for the agnostic online setting. Additionally we achieve an improved mistake bound $\scO(\omega(G)+\log(n/\omega(G))$ using Halving and a runtime of $2^{\omega(G)}\poly(n)$.

\item As a byproduct of the results above, we resolve the open case $k=2$ of the partitioning problem of~\cite{gonzalez2020covering}, which showed that for $k \ge 3$ it is NP-hard to decide whether the vertex set $V$ of a graph $G$ admits a $k$-partition $V_1,\ldots,V_k$ where each $V_i$ is monophonically convex. We also prove that $|\scHm| \le 2^{\omega(G)} \poly(n)$, significantly improving on the only previously known and trivial bound $|\scHm| = \scO(n^d)$, where $d=\scO(\omega(G))$ is the VC dimension of $\scHm$.

\end{enumerate}

\bigskip\noindent
From a technical point of view, we rely on several technical results on the structure of m-halfspaces, including characterizations of their cutsets and efficient computation of so-called ``shadows''.
It should be noted that ours are among the few \emph{constructive} results on the efficiency of learning abstract halfspaces (or convex concepts) in graphs. 
Our work belongs to a line of research on graph hypothesis spaces, their learning properties, and fixed-parameter learnability~\citep{chepoi2007covering,chepoi2021labeled,chalopin_et_al:LIPIcs.MFCS.2022.31,le2023vc,brand2023parameterized}. Note that our runtime bounds for agnostic PAC learning and Halving are indeed polynomial for any family of graphs with bounded $\omega(G)$. This is for example the case for triangle-free graphs or planar graphs and a much weaker condition than more commonly used parameters such as treewidth \citep{thiessen2021active} and maximum \emph{clique-minor} size \citep{duchet1983ensemble,chalopin_et_al:LIPIcs.MFCS.2022.31,le2023vc} to achieve polynomial runtime or bounds on the VC dimension.

\paragraph{Organization of the manuscript.} Section~\ref{sec:related} reviews related work. Sections~\ref{sub:supervised}-\ref{sub:online} discuss our main results, their significance, and their technical ingredients. Section~\ref{sec:sparse} presents structural properties of monophonic halfspaces that are needed by our algorithms. Section~\ref{sec:polycheck} presents our polynomial-time consistency checker, and Section~\ref{sec:fpt} presents our empirical risk minimization (ERM) algorithm. All missing details and proofs can be found in the appendix.

\section{Further Related Work}\label{sec:related}
The complexity of consistency checking for monophonic halfspaces was unknown. \cite{gonzalez2020covering} show that deciding whether $V$ can be partitioned into $k \ge 3$ nonempty m-convex sets is NP-hard. They leave open the case $k=2$, i.e., deciding if there is a nontrivial m-halfspace of $G$, or in other words whether $\scHm(G) \ne \emptyset$. Our polynomial-time consistency checker proves that the case $k=2$ is in polynomial time: for every pair $\{u,v\} \in \binom{V}{2}$ use the checker to verify whether there is a halfspace that separates $u$ from $v$. This result was recently independently achieved by \citet{elaroussi2024half} and \citet{chepoi2024separation}. This result should be contrasted with consistency checking in geodesic halfspaces, which is known to be NP-hard~\citep{seiffarth2023maximal}. As monophonic halfspaces are a sub-family of geodesic halfspaces, our result can be seen as pushing the boundary of tractability of halfspace separation problems on graphs, see also \citet{chepoi2024separation} and \citet{elaroussi2024half}.

Empirical risk minimization can be reduced to listing the concept class $\scHm(G)$.
\cite{duchet1988convex} observes that, if $H$ is m-convex, then the vertices of $\overline H$ that are adjacent to $H$ form a clique. Therefore one can list $\scHm(G)$ by listing all pairs of cliques of $G$ and checking if the edges between them form a cut of $G$, for a running time of $n^{2\omega(G)}\poly(n)$. A better bound can be achieved if one is given a polynomial-time consistency checker: in that case, by a folklore algorithm one can list $\scHm(G)$ in time $|\scHm(G)| \poly(n)$. Our work gives both a polynomial-time consistency checker and a tight bound on $|\scHm(G)|$; neither one was known before. In particular, bounds on $|\scHm(G)|$ given by standard VC-dimension arguments suffer an exponential dependence on the cutsize $c$ (i.e., the number of edges) of the halfspace~\citep{kleinberg2004detecting}. In our case $c$ can be as large as $\Theta(\omega(G)^2)$, which yields $|\scHm| \le n^{\scO(\omega(G)^2)}$.
This is significantly beaten by our novel bound $|\scHm| \le \frac{4m \, 2^{\omega(G)}}{\omega(G)}+2$.
\citet{glantz2017finding} give polynomial time algorithms for enumerating geodesic halfspaces of bipartite and planar graphs, but do not have results for general graphs. By contrast, we can enumerate all m-halfspaces in optimal time $|\scHm|\poly(n)$ up to polynomial factors.

For active learning, \citet{thiessen2021active} give lower and upper bounds, but for \emph{geodesic} halfspaces. Their algorithm requires computing the geodesic hull number, which is APX-hard and without constant-approximation algorithms; our algorithm runs in polynomial time. \citet{bressan2021exact} also studied active learning on graphs under a geodesic convexity assumption. They achieved polynomial time however with additional assumptions on the convex sets, such as margin.
For online learning, \citet{thiessen2022online} give again results for geodesic halfspaces. Their algorithms, however, are computationally inefficient and/or loose in terms of mistakes. We instead rewrite monophonic halfspaces as a union of a small number of m-shadows, which enables an algorithm (Winnow) that is both computationally efficient and near-optimal in terms of mistakes.

\section{Preliminaries}\label{sec:prelim}

Let $G=(V,E)$ be a simple undirected graph and let $n=|V|$ and $m=|E|$.
Without loss of generality we assume $G$ is connected.\footnote{If $G$ is not connected then it has no m-halfspace, unless it consists of precisely two connected components, which are then the only two halfspaces.}
For any $v \in V$ let $N(v) = \{u \in V \,|\, \{u,v\} \in E\}$. The \emph{cutset} induced by $X \subseteq V$ is $\Cut(X) = \{\{u,v\} \in E \,|\, u \in X, v \notin X\}$. For an edge $\{u,v\} \in \Cut(X)$ we may write $uv \in \Cut(X)$ to specify that $u \in X$. The \emph{border} of $X$ is $\Border(X) = \{u \,|\, uv \in \Cut(X)\}$, the set of vertices of $X$ with a neighbour in $V\setminus X$. We denote by $\overline G$ or $\neg G$ the complement $(V,{V \choose 2} \setminus E)$ of $G$. We let $\omega(G)$ and $\alpha(G)$ be respectively the clique number and the independence number of $G$, and $\omgtil(G)=\max\{\omega(G),3\}$.
For any $X \subseteq V$ we denote by $G[X]$ the subgraph of $G$ induced by $X$.
If $P$ is a path in $G$, any edge in $E(G)\setminus E(P)$ is called a \emph{chord}; a path is induced if and only if it has no chords. Any shortest path is an induced path. We denote by $\diam_g(G)$ the diameter of $G$, and by $\diam_m(G)$ the maximum number of edges in any induced path. If $a,b \in V$ and $X \subseteq V$, then $X$ is an $(a,b)$-\emph{separator} if in $G$ every path between $a$ and $b$ intersects $X$. An algorithm is FPT (fixed-parameter tractable) with respect to some parameter $k$ (in our case $k=\omega(G)$) if its runtime is bounded by $f(k)\poly(n)$ for some computable function $f$.

\begin{figure}[h]
    \centering
    \colorlet{ffttww}{orange}
\definecolor{ududff}{rgb}{0.30196078431372547,0.30196078431372547,1}
\begin{tikzpicture}[line cap=round,line join=round,>=triangle 45,x=1cm,y=1cm,scale=.7]
\draw [line width=1pt] (-2,0)-- (0,-1);
\draw [line width=1pt] (-2,-0.89835)-- (0,-1);
\draw [line width=1pt] (-2,-2)-- (0,-1);
\draw [line width=1pt] (0,-1)-- (1.6941,-1.15245);
\draw [line width=1pt] (2.6742,-0.09975)-- (3.5454,-0.86205);
\draw [line width=1pt] (3.5454,-0.86205)-- (1.6941,-1.15245);
\draw [line width=1pt] (2.6742,-0.09975)-- (2.9646,-2.38665);
\draw [line width=1pt] (2.9646,-2.38665)-- (1.6941,-1.15245);
\draw [line width=1pt] (2.6742,-0.09975)-- (1.6941,-1.15245);
\draw [line width=1pt] (2.9646,-2.38665)-- (3.5454,-0.86205);
\draw [line width=1pt] (3.5454,-0.86205)-- (6,0);
\draw [line width=1pt] (6,0)-- (4.707,-1.47915);
\draw [line width=1pt] (4.707,-1.47915)-- (6,-2);
\draw [line width=1pt] (6,0)-- (6,-2);
\draw [line width=1pt] (6,-2)-- (8,-2);
\draw [line width=1pt] (6,0)-- (8,0);
\draw [line width=1pt] (8.9541,-0.24495)-- (8,0);
\draw [line width=1pt] (10,0)-- (8.9541,-0.24495);
\draw [line width=1pt] (10.0068,-0.68055)-- (8.9541,-0.24495);
\draw [line width=1pt] (8,0)-- (8,-2);
\draw [line width=1pt] (10,0)-- (10.0068,-0.68055);
\draw [line width=1pt] (8,-2)-- (8.8089,-1.58805);
\draw [line width=1pt] (8.8089,-1.58805)-- (10,-1.37025);
\draw [line width=1pt] (8.8089,-1.58805)-- (10,-2);
\draw [line width=1pt] (10,-1.37025)-- (10,-2);
\draw [line width=1pt] (3.5454,-0.86205)-- (4.707,-1.47915);
\draw [line width=1pt] (2.9646,-2.38665)-- (4.707,-1.47915);
\draw [line width=1pt] (2.6742,-0.09975)-- (6,0);
\draw [line width=1pt] (2.9646,-2.38665)-- (6,-2);
\draw [line width=1pt] (1.6941,-1.15245)-- (4.707,-1.47915);
\begin{scriptsize}
\draw [fill=ududff] (3.5454,-0.86205) ++(-5pt,0 pt) -- ++(5pt,5pt)--++(5pt,-5pt)--++(-5pt,-5pt)--++(-5pt,5pt);
\draw [fill=ududff] (2.6742,-0.09975) ++(-5pt,0 pt) -- ++(5pt,5pt)--++(5pt,-5pt)--++(-5pt,-5pt)--++(-5pt,5pt);
\draw [fill=ududff] (1.6941,-1.15245) ++(-5pt,0 pt) -- ++(5pt,5pt)--++(5pt,-5pt)--++(-5pt,-5pt)--++(-5pt,5pt);
\draw [fill=ududff] (2.9646,-2.38665) ++(-5pt,0 pt) -- ++(5pt,5pt)--++(5pt,-5pt)--++(-5pt,-5pt)--++(-5pt,5pt);
\draw [fill=ffttww]  (6,-2) ++(-5pt,0 pt) -- ++(5pt,5pt)--++(5pt,-5pt)--++(-5pt,-5pt)--++(-5pt,5pt);
\draw [fill=ffttww] (6,0) ++(-5pt,0 pt) -- ++(5pt,5pt)--++(5pt,-5pt)--++(-5pt,-5pt)--++(-5pt,5pt);
\draw [fill=ffttww] (4.707,-1.47915) ++(-5pt,0 pt) -- ++(5pt,5pt)--++(5pt,-5pt)--++(-5pt,-5pt)--++(-5pt,5pt);
\draw [fill=ududff] (3.5454,-0.86205) node (5pt) {};
\draw [fill=ududff] (2.6742,-0.09975) node (5pt) {};
\draw [fill=ududff] (1.6941,-1.15245) node (5pt) {};
\draw [fill=ududff] (2.9646,-2.38665) node (5pt) {};
\draw [fill=ffttww] (6,0) node (5pt) {};
\draw [fill=ffttww] (4.707,-1.47915) node (5pt) {};
\draw [fill=ffttww] (6,-2) node (5pt) {};
\draw [fill=ffttww] (8,-2) circle (5pt);
\draw [fill=ffttww] (8,0) circle (5pt);
\draw [fill=ffttww] (8.9541,-0.24495) circle (5pt);
\draw [fill=ffttww] (10,0) circle (5pt);
\draw [fill=ffttww] (10.0068,-0.68055) circle (5pt);
\draw [fill=ffttww] (8.8089,-1.58805) circle (5pt);
\draw [fill=ffttww] (10,-1.37025) circle (5pt);
\draw [fill=ffttww] (10,-2) circle (5pt);
\draw [fill=ududff] (0,-1) circle (5pt);
\draw [fill=ududff] (-2,-0.89835) circle (5pt);
\draw [fill=ududff] (-2,0) circle (5pt);
\draw [fill=ududff] (-2,-2) circle (5pt);
\end{scriptsize}
\end{tikzpicture}
    \caption{A toy graph $G$ whose vertex set $V$ is partitioned into a monophonic halfspace $H$ (in blue) and its complement $\bar H$ (in orange). The diamond-shaped vertices form the borders $\Border(H)$ and $\Border(\bar H)$, see below. Both $G[\Border(H)]$ and $G[\Border(\bar H)]$ are cliques, see Lemma~\ref{lem:H_charact}.}
    \label{fig:halfspace}
\end{figure}
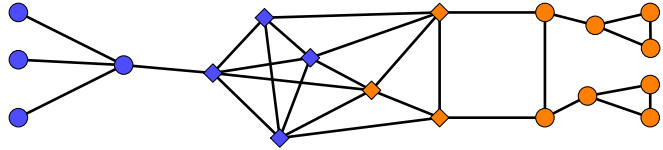

Let $V$ be a set and $\scC \subseteq 2^V$. The set system $(V,\scC)$ is a \emph{convexity space} if (i) $\emptyset, V \in \scC$ and (ii) $(\bigcap_{C\in\scC'}C)\in\scC$ for every $\scC'\subseteq\scC$.\footnote{When $V$ is not finite one needs additional constraints, see~\citet{van1993theory}; clearly this is not our case.} A set $C \subseteq V$ is said to be convex if $C \in \scC$. Convexity spaces abstract standard Euclidean convexity (see, e.g., \citet{van1993theory}). The \emph{convex hull} of $A \subseteq V$ is $\conv(A) = \bigcap_{C\in\scC:A\subseteq C} C\in\scC$. The \emph{hull number} of $(V,\scC)$ is the size of the smallest $A$ such that $\conv(A)=V$.
A set $R \subseteq V$ is \emph{Radon-independent} if $\conv(R')\cap\conv(R\setminus R')=\emptyset$ for all $R'\subseteq R$. The \emph{Radon number} of $(V,\scC)$ is the size of its largest Radon-independent set.
Many convexity spaces are defined by \emph{intervals}. A map $I\colon V\times V \rightarrow 2^V$ is an interval map if (i) $u,v\in I(u,v)$ and (ii) $I(u,v)=I(v,u)$ for all $u,v\in V$. An example is the geodesic interval of a metric space $(V,d)$ defined by $I_d(u,v) = \{z\in V\mid d(u,z)+d(z,v)=d(u,v)\}$. Any interval map $I$ defines a convexity space $(V,\scC)$ where $C \in \scC$ if and only if $I(u,v)\subseteq C$ for all $u,v\in C$. A \emph{graph convexity space} for $G=(V,E)$ is any convexity space $(V,\scC)$ where $G[C]$ is connected for all $C\in\scC$~\citep{pelayo2013geodesic}. We denote by $r(G)$ the Radon number and by $h(G)$ the hull number of $(V,\scC)$.

For any distinct $u,v \in V$ the \emph{monophonic interval} (m-interval) $\MonInt(u,v)$ between $u$ and $v$ is the set of all vertices that lie on some induced path between $u$ and $v$; while $\MonInt(u,u)=\emptyset$ for all $u\in V$. A set $X \subseteq V$ is \emph{monophonically convex} (m-convex) if $\MonInt(u,v) \subseteq X$ for all $u,v \in X$. A set $X \subseteq V$ is a \emph{monophonic halfspace} (m-halfspace), if both $X$ and its complement $\overline X = V \setminus X$ are m-convex. We denote by $\scHm(G)$ the set of all m-halfspaces of $G$. The \emph{monophonic shadow} (m-shadow) of $u$ with respect to $v$ is the set $u/v = \{z \in V : u \in \MonInt(z,v)\}$. It is known that computing $\MonInt$ is NP-hard in general~\citep{dourado2010complexity}; as $\MonInt(u,v) = \{z \in V : u\in z/v\}$, computing m-shadows is NP-hard, too.
Note that if $X \subseteq V$ is m-convex and $vu \in \Cut(X)$ then $u/v \subseteq \overline{X}$. The \emph{monophonic convex hull} (m-hull) of $X \subseteq V$, denoted $\conv(X)$, is the smallest m-convex set $X'$ such that $X \subseteq X' \subseteq V$. It is known that one can compute $\conv(X)$ in time $\poly(n)$ for any given $X$~\citep{dourado2010complexity}.
To discuss some tightness results of our achieved bounds we rely on so-called \emph{separation axioms}, which determine the separability of vertices by halfspaces.
\begin{definition}[\citet{van1993theory}]\label{def:sep_axioms}
	A convexity space $(X, \scC)$ is:
	\begin{itemize}[itemsep=2pt,parsep=0pt,topsep=4pt]
		\item $S_1$ if and only if each singleton $x \in X$ is convex.
		\item $S_2$ if and only if each pair of distinct elements $x,y \in X$ is halfspace separable.
		\item $S_3$ if and only if each convex set $C \in \scC$ and elements $x \in X\setminus C$ are halfspace separable.
		\item $S_4$ if and only if any two disjoint convex sets are halfspace separable.
	\end{itemize}
\end{definition}
By slightly abusing notation, we write $S_i$ for the family of graphs whose m-convexity space is $S_i$ (for $i=1,\dots,4$). All graphs are in $S_1$, and $S_2\supsetneq S_3 \supsetneq S_4$. The graph families $S_2$, $S_3$, and $S_4$ are exactly characterized \citep{jamison1981convexity,bandelt1989graphs}, see also \citet{chepoi1994separation,chepoi2024separation}.

\section{Supervised Learning}\label{sub:supervised}
We study the standard PAC (probably approximately correct) setting with a given finite instance space $V$ and a known hypothesis space $\scH\subseteq 2^V$. 
In our case $\scH=\scHm(G)$ is defined implicitly by the input graph $G$.
A \emph{labeled sample} is a pair $(S,y)$ where $S\subseteq V$ is a finite (multi)set of vertices and $y\colon S \to \{0,1\}$ is a binary labeling of the sample.
We say that it is possible to $(\varepsilon, \delta)$-PAC learn $(V, \scHm)$ if we can provide an m-halfspace $H=H(S,y)$ that has error bounded by $\varepsilon$ with probability at least $1-\delta$, for any $\varepsilon,\delta\in(0,1)$, under some unknown joint distribution over vertices and their labels.
The aim of PAC learning is to guarantee this property with the smallest sample complexity possible.
Note that, while standard PAC learning formulations typically measure the running time as a function of the sample size (see, e.g., \citet{valiant1984theory,shalev2014understanding}), we measure running time as a function of $n$. Note that a dependence on $n$ is unavoidable: even just to check whether two vertices $u,v \in V$ are separated by some $H \in \scHm(G)$, one needs to distinguish between $G$ being a path or a cycle, which takes time $\Omega(n)$.
We are in the \emph{realizable setting} if for any given sample $(S,y)$ there exists a $H\in\scHm$ such that for all $s\in S$ it holds that $y(s)=1$ if and only if $s\in H$. Our main result for realizable PAC learning is:
\begin{restatable}[Realizable PAC is poly-time]{newthm}{realizablePACpoly}\label{thm:PAC_realizable}
    There exists an algorithm that, for any $\varepsilon,\delta \in (0,1)$, can $(\varepsilon,\delta)$-PAC learn $(V,\scHm)$ in the realizable setting using $\scO\Bigl(\frac{\omega(G)\log(1/\varepsilon)+\log(1/\delta)}{\varepsilon}\Bigr)$ labeled samples and time $\poly(n)$.
\end{restatable}
The bound on the sample complexity of our algorithm follows by a bound on the VC dimension of $\scHm$ in terms of $\omega(G)$ (Proposition~\ref{pro:VC_mono}), coupled with standard PAC learning bounds~\citep{blumer1989learnability}. The bound is near-optimal save possibly for the $\log(1/\varepsilon)$ factor (see below).
The bound on the running time is instead given by Theorem~\ref{thm:polycheck} (see below) which provides a polynomial-time algorithm for computing an m-halfspace consistent with a given realizable sample.
Similarly to the realizable case, we obtain the following result for agnostic PAC learning:
\begin{restatable}[Agnostic PAC is FPT]{newthm}{agnosticPACFPT}\label{thm:PAC_agnostic}
    There exists an algorithm that, for any $\varepsilon,\delta \in (0,1)$, can $(\varepsilon,\delta)$-PAC learn $(V,\scHm)$ in the agnostic setting using $\scO\Bigl(\frac{\omega(G)+\log(1/\delta)}{\varepsilon^2}\Bigr)$ labeled samples and time $|\scHm|\poly(n)\leq 2^{\omega(G)}\poly(n)$.
\end{restatable}
Similarly to the realizable case, the sample complexity bound follows by coupling VC dimension bounds with standard PAC bounds, and is near-optimal. The result follows by using the ERM from Corollary~\ref{cor:erm-fpt}, which uses the version space enumeration of Theorem~\ref{thm:polydelayversionspace} (see below).

To prove that the sample complexity bounds of Theorem~\ref{thm:PAC_realizable} and Theorem~\ref{thm:PAC_agnostic} are near-optimal, it suffices to note that the VC dimension of $\scHm$ is $\omega(G)$ for clique graphs. We prove that this is true for all $S_4$ graphs (see Section~\ref{sec:prelim}), and essentially optimal. Let $\omgtil(G)=\max\{\omega(G),3\}$. Then:
\begin{restatable}{prop}{monovc}\label{pro:VC_mono}\label{prop:vc_s4_lower}
If $G\in S_4$ then $\VC(\scHm(G))\geq\omega(G)$.
Moreover, $\VC(\scHm(G))\le\omgtil(G)$ for all $G$.
\end{restatable}

\subsection{Polynomial-Time Consistency Checking and Version Space Enumeration}

The algorithms of Theorem~\ref{thm:PAC_realizable} and Theorem~\ref{thm:PAC_agnostic} rely on the standard empirical risk minimization (ERM) approach: given a labeled sample $(S,y)$, we find a hypothesis $H \in \scHm$ that has minimal disagreement with $(S,y)$. This, we need efficient algorithms for ERM. In the realizable case, where $(S,y)$ is realizable, we need an algorithm to find a hypothesis consistent with the sample. Formally, an m-halfspace $H\in\scHm$ is \emph{consistent} with a given labeled sample $(S,y)$  if $y^{-1}(0) \subseteq \overline H$ and $ y^{-1}(1) \subseteq H$. %
Our main technical contribution is $\PolyChecker$, an algorithm that gives:
\begin{restatable}[Consistency check is in polynomial time]{newthm}{polyconsistency}\label{thm:polycheck}
Let $G=(V,E)$ and $(S,y)$ a labeled sample. Then, $\PolyChecker(G, (S,y))$ (Algorithm~\ref{alg:polychecker}) runs in time $\poly(n)$ and returns an $H \in \scHm(G)$ consistent with $(S,y)$, or reports that no such $H$ exists.
\end{restatable}
We note that \citet{elaroussi2024half} and \citet{chepoi2024separation} independently obtained an analogous result recently using different techniques.

The intuition behind \PolyChecker is to guess an edge in the cut of $H$, and then construct a set of constraints that reduces the search space to precisely those subsets $H \subseteq V$ that are m-halfspaces and that are consistent with $(S,y)$. It turns out that, by carefully choosing those constraints, one actually obtains in time $\poly(n)$ an instance of 2-satisfiability, which in turn is solvable in time $\poly(n)$. Choosing the constraints requires defining several specific subsets of vertices, as well as proving several structural results on m-halfspaces and their cutsets.
Another computational ingredient is showing that computing m-shadows over edges is easy, although hard in general.
\begin{restatable}{lem}{edgeShadowsPoly}\label{lem:shadows_in_P}
Let $\{z,v\} \in E$. Then, $z/v = \{x \in V \mid N(v) \setminus \{z\} \text{ is not a } (z,x)\text{-separator in } G\}$. As a consequence one can compute $z/v$ in time $\poly(n)$.
\end{restatable}

As an immediate byproduct of Theorem~\ref{thm:polycheck} we answer an open problem of \citet{gonzalez2020covering}, who proved that deciding if $V$ admits a proper $k$-partition into m-convex sets is NP-hard for $k \ge 3$.
\begin{corollary}
    In time $\poly(n)$ one can decide if $V$ admits a proper $2$-partition into m-convex sets in~$G$ (that is, if there exists $H\in\scHm$ such that $\emptyset\subsetneq H \subsetneq V$).
\end{corollary}
A detailed description of \PolyChecker and a sketch of the proof of Theorem~\ref{thm:polycheck} are given in Section~\ref{sec:polycheck}.
Using \PolyChecker, we obtain an algorithm for ERM. For a given labeled sample $(S,y)$, the \emph{version space} $\VS((S,y),\scHm)$ consists of all $H \in \scHm$ consistent with $(S,y)$. Thus, one can find an empirical risk minimizer for $(S,y)$ by just listing $\VS=\VS((S,y),\scHm)$ and picking a m-halfspace with minimum disagreement. We say that an algorithm lists $\VS$ \emph{with delay $t$} if it spends time at most $t$ to output the next element of $\VS$. A listing algorithm is considered to be efficient if it has polynomial delay, that is, $t=\poly(n)$; this also implies that $\VS$ is listed in total time $|\VS| \poly(n)$. Using \PolyChecker\ and a folklore technique, we prove that such an algorithm exists.
\begin{restatable}[Listing the version space with polynomial delay]{newthm}{polydelayversionspace}\label{thm:polydelayversionspace}
	There exists an algorithm \PolyDelay\ that, given a graph $G$ and a labeled sample $(S,y)$, lists $\VS=\VS(S,\scHm)$ with delay $\poly(n)$ and total time $|\VS|\poly(n)$.
\end{restatable}
As a last ingredient for our ERM bounds, we bound the size of $\scHm(G)$. To this end we develop another listing algorithm, Algorithm~\ref{alg:fpt_enum}, which by a careful analysis yields the following upper bound. Further details on how we accomplish this are deferred to Section~\ref{sec:fpt}.
\begin{restatable}{newthm}{halfspaceCount}\label{thm:halfspaces_count}
$|\scHm(G)|\leq \frac{4m 2^{\omega(G)}}{\omega(G)}+2$ for all graphs $G$.
\end{restatable}
Together this yields the fact that ERM is FPT for m-halfspaces.
\begin{restatable}[ERM is FPT]{cor}{ermfpt}\label{cor:erm-fpt}
    Given $G=(V,E)$ and a labeled sample $(S,y)$, one can compute $H_{\mathrm{ERM}} \in \argmin_{H\in\scHm} \frac1{|S|} \sum_{s \in S} \mathbbm{1}\{y(s) \ne \mathbbm{1}_{s\in H}\}$ in time $|\scHm|\poly(n)\leq$ $2^{\omega(G)}\poly(n)$.
\end{restatable}

\section{Active Learning}\label{sub:active}
In the active learning setting, the algorithm is given a graph $G=(V,E)$, and nature selects a concept $H\in\scHm(G)$. The algorithm can query any vertex $x \in V$ for its label $\mathbbm{1}_{x \in H}$. The goal of the algorithm is to output $H$ by making as few queries as possible.
This problem is a special case of realizable transductive active learning on a graph \citep{afshani2007complexity, guillory2009label,cesa2010active,dasarathy2015s2}, and can be seen as a variant of query learning \citep{angluin1988queries, hegedHus1995generalized} with an additional fixed graph given where vertices correspond to the instance space.
Now let $\scH\subseteq 2^V$ be a concept class. The \emph{query complexity} of $(V,\scH)$ is the maximum number of queries an optimal algorithm makes over $H \in \scH$. More precisely, for any algorithm $A$ and any $H \in \scH$, let $\qc(A,H)$ denote the number of queries $A$ make on $G$ when nature chooses $H$. The query complexity of $A$ on $\scH$ is $\qc(A,\scH)=\max_{H\in\scH} \qc(A,H)$. The query complexity of $\scH$ is $\qc(\scH)=\min_A \qc(A,\scH)$. 
Our main result is stated as follows.

\begin{restatable}[Poly-time active learning]{newthm}{polyactive}\label{thm:active_upper_bound}
    It is possible to actively learn $(V,\scHm)$ in time $\poly(n)$ with query complexity $\scO\left(h(G)+\log\diam_g(G)+\omega(G)\right)$. 
\end{restatable}
The full proof of Theorem~\ref{thm:active_upper_bound} is given in Appendix~\ref{apx:active}; here we provide some intuition. First, we compute a minimum monophonic hull set in polynomial time relying on \citet{dourado2010complexity}. If any cut-edge $\{u,v\}$ exists we can then find it with $\scO(\log \diam_g(V))$ queries. Using some structural results (Lemma~\ref{lem:all_shadows} and Lemma~\ref{lem:clique_connectedcomps}, see Section~\ref{sec:sparse}) we can infer the remaining halfspace by querying at most $\omega(G)$ from the set $\triangle_{uv}= N(u) \cup N(v) \cup \{u,v\}$. The required m-shadows can be computed efficiently  by Lemma~\ref{lem:shadows_in_P}. Theorem~\ref{thm:active_upper_bound} should be contrasted with the active learning algorithm provided by \citet{thiessen2021active} for \emph{geodesic} halfspaces; that algorithm does not guarantee polynomial running time, as it requires solving the minimum geodesic hull set problem, which is APX-hard.

Along the previously mentioned separation axioms from Definition~\ref{def:sep_axioms} we achieve increasingly tighter lower bounds on the query complexity, eventually matching our algorithmic upper bound from Theorem~\ref{thm:active_upper_bound} for all $S_4$ graphs, the strongest separability assumption. 
\begin{restatable}{newthm}{activeLowerBounds} \label{thm:lower_bounds}
Let $G$ be a graph. The following holds for the query complexity $\qc(\scHm(G))$: 
\begin{itemize}[itemsep=2pt,parsep=0pt,topsep=4pt]
	\item if $G\in S_2$, then $\qc(\scHm(G))\geq \Omega(\log \diam_m(G))$, 
	\item if $G\in S_3$, then  $\qc(\scHm(G))\geq \Omega(\log \diam_m(G)+h(G))$, and
	\item if $G\in S_4$, then $\qc(\scHm(G))\geq \Omega(\log \diam_m(G)+ h(G)+\omega(G))$.
\end{itemize}    
\end{restatable}

\section{Online Learning}\label{sub:online}
The classical (realizable) online learning problem of \citet{littlestone1988learning} can be modelled as an iterative game between a learner and the environment over a finite number $T$ of rounds. The instance space $V$ and a hypothesis space $\scH\subseteq 2^V$ is known and fixed. First, the environment chooses a hypothesis $H$ from a hypothesis space $\scH$. Then, in each round $t=1,\dots,T$:
\begin{enumerate}[itemsep=2pt,parsep=0pt,topsep=4pt]
    \item the environment chooses a point $v_t\in V$,
    \item the learner predicts the label $\hat{y}_t\in\{0,1\}$ of $v_t$,
    \item the environment reveals the true label $y_t=\mathbbm{1}_{v_t\in H}$,
    \item the learner made a mistake if $\hat{y}_t \neq y_t$.
\end{enumerate}
The goal of the learner is to minimize the total number of mistakes. More precisely let $A$ be an algorithm for this online learning problem. Then, let $M(A,H)$ for $H\in\scH$ denote the worst-case number of mistakes $A$ would make on any sequence labeled by $H$ over $T$ rounds.
The mistake bound of $A$ on $\scH$ is thus defined as $M(A,\scH)=\max_{H\in\scH} M(A,H)$. We are interested in the optimal mistake bound $M(\scH)=\min_A M(A,\scH)$, also known as the Littlestone dimension of $\scH$. This setting can be extended to the agnostic/non-realizable case, see \citet{ben2009agnostic}.

The node classification (or node labeling) variant of this problem is well studied \citep{herbster2005online,cesa2013random,herbster2015online}.
As in the active learning variant, the main parameter is the (potentially effective resistance weighted) cutsize, linearly determining the mistake bounds.
In this section, we mainly study the variant of the above realizable online learning problem over our hypothesis class $\scHm$.

\begin{restatable}[Poly-time online learning]{newthm}{polyonline}\label{thm:polyAndFPTonline}
    Realizable online learning of monophonic halfspaces is possible in time $\poly(n)$ with a mistake bound of $\scO(\omega(G)\log n)$. It is also possible in time $|\scHm|\poly(n)\leq 2^{\omega(G)}\poly(n)$ with a mistake bound of $\scO\Bigl(\omega(G) + \log \frac{n}{w(G)}\Bigr)$.
\end{restatable}
The first mistake bound is achieved by the Winnow algorithm \citep{littlestone1988learning}  and relies on a novel representation of m-halfspaces as a sparse disjunction of m-shadows (Lemma~\ref{lem:atmost_clique_cutedges}).
The second mistake bound is achieved by the Halving algorithm~\citep{barzdin1972prediction,littlestone1988learning}, together with our version space listing algorithm (Theorem~\ref{thm:polydelayversionspace}) and our upper bound on $|\scHm|$ (Theorem~\ref{thm:halfspaces_count}). 
In Appendix~\ref{sec:agnostic_online} we additionally discuss agnostic online learning of m-halfspaces relying again on our decomposition by Lemma~\ref{lem:atmost_clique_cutedges} and known results for Winnow \citep{blum1996online}. We achieve the following mistake bound efficiently almost matching standard expert-based techniques \citep{cesa1997use,ben2009agnostic}, which typically require $\scO(|\scHm|)$ time per round.
\begin{theorem}\label{thm:agnosticwinnow}
    A mistake bound $\scO\bigl(\omega(G) (M^* + \log n)\bigr)$ for agnostic online learning over $\scHm(G)$ is possible in time $\poly(n)$, where $M^*$ is the mininum number of mistakes of any m-halfspace.
\end{theorem}

\paragraph{Lower bounds.}
 The achieved mistake bounds are near-optimal by the following result, whose proof is in Appendix~\ref{apx:online_winnow}.
\begin{restatable}{prop}{mistakelowerbound} \label{prop:online_lower_bound}
    For any $S_4$ graph $G$, it holds that $M(\scHm(G))\geq \omega(G)$.
\end{restatable}

\section{Structural Lemmas regarding Monophonic Halfspaces}\label{sec:sparse}
This section provides some structural results on monophonic halfspaces used by our algorithms. Loosely speaking, the main message is that monophonic halfspaces can be expressed as unions of a small number of ``simpler'' sets, whose labels can be inferred from a small number of labeled vertices.
First, however, we give a basic result about the border of the cut.
\begin{restatable}{lem}{borderAreCliques}\label{lem:H_charact}
    $H \in \scHm(G)$ if and only if $G[\Border(H)]$ and $G[\Border(\overline H)]$ are cliques.
\end{restatable}
The statement follows by \citet[Proposition 4.2]{duchet1988convex} applied to both $H$ and $\overline{H}$; a self-contained proof can be found in Appendix~\ref{apx:fpt}. Figure~\ref{fig:halfspace} shows an example.
For an edge $uv \in E$ let $\triangle_{uv} = (N(u) \cap N(v)) \cup \{u,v\}$. 
Our first result shows that every halfspace $H\in\scHm$ can be expressed as the union of m-shadows over $H \cap \triangle_{uv}$, for every $uv \in \Cut(H)$.
\begin{restatable}{lem}{allShadows}\label{lem:all_shadows}
Let $H \in \scHm(G)$, let $uv \in \Cut(H)$, and let $\triangle_{uv}=(N(u)\cap N(v)) \cup \{u,v\}$. Then, $H$ consists precisely of those vertices whose label can be inferred from $u,v$, and a vertex in $H \cap \triangle_{uv}$. Formally:
\[
H = \bigcup_{z \in H \cap \triangle_{uv}} z/v \enspace.
\]
\end{restatable}
Note that Lemma~\ref{lem:all_shadows} does not yet give a \emph{sparse} representation, in the sense that it does not bound the number of m-shadows $|H \cap \triangle_{uv}|$ used to obtain $H$. Combining Lemma~\ref{lem:all_shadows} with Lemma~\ref{lem:H_charact}, however, we can prove that $H$ can be expressed as the union of at most $\omega(G)$ such m-shadows.
\begin{lemma}\label{lem:atmost_clique_cutedges}
    Let $H \in \scHm(G)$. Then, there exists a subset $C\subseteq \Cut(H)$ with $|C|\leq \omega(G)$ such that \[H = \bigcup_{\{z,v\}\in C} z/v \enspace.\]
\end{lemma}
Lemma~\ref{lem:atmost_clique_cutedges} will be crucial to enable efficient online learning using Winnow.
We shall also note the following fact, used by our bounds on $|\scHm|$ and by our polynomial-time active learning algorithm.
\begin{restatable}{lem}{cliqueConnectedComps}\label{lem:clique_connectedcomps}
    Let $H\in\scHm$ and $uv$ be a cut-edge of $H$. Then, the number $N_{uv}$ of connected components in $\neg G[\triangle_{uv}]$ satisfies $N_{uv} \le \omega(G)$.
\end{restatable}

\section{A Polynomial-Time Consistency Checker}
\label{sec:polycheck}
This section describes our main technical contribution: \PolyChecker\ (Algorithm~\ref{alg:polychecker}), an algorithm to find an m-halfspace consistent with a given a labeled sample if one exists.
We show that such an algorithm runs in polynomial time; this proves Theorem~\ref{thm:polycheck}, which we restate for convenience.
\polyconsistency*

Let us illustrate the idea behind \PolyChecker. Let $(S,y)$ be the labeled sample, and suppose a halfspace $H$ consistent with $(S,y)$ exists. For every $uv \in E$, we guess that $uv$ in is the cut $\Cut(H)$, and we compute the set $\triangle_{uv} = (N(u) \cap N(v)) \cup \{u,v\}$. By exploiting the results of Section~\ref{sec:sparse}, we know that from $\triangle_{uv} \cap H$ one can infer the rest of $H$. Unfortunately, there seems to be no easy way of learning $\triangle_{uv} \cap H$, apart from exhaustively guessing all subsets of $\triangle_{uv}$. Thus, we need a more sophisticate strategy. We use the results of Section~\ref{sec:sparse}, as well as the closure properties of monophonic halfspaces, to construct a set of constraints on the space of all possible subsets $H \subseteq V$. The hope is that those constraints identify precisely the subsets $H \subseteq V$ that satisfy $H \in \scHm(G)$ and are consistent with $(S,y)$.

Formally, we will perform a reduction to 2-SAT.
For every orientation $uv$ of every edge of $G$, we construct in time $\poly(|V|)$ a 2-SAT formula $\satvar_{uv}$ whose literals are in the form ``$x \in H$'' or ``$x \notin H$'', where $x \in V$. Any $H \subseteq V$ yields an evaluation of $\satvar_{uv}$ in the obvious way. We can prove that the solutions to $\satvar_{uv}$ are precisely the m-halfspaces $H\in\scHm(G)$ such that $uv \in \Cut(H)$. Constructing the 2-SAT formula $\satvar_{uv}$ in polynomial time is the challenging part of the algorithm. Once we have computed $\satvar_{uv}$, it is straightforward to extend it to a 2-SAT formula $\satvar_{uv}(S,y)$ that is satisfied precisely by the solutions to $\satvar_{uv}$ that are consistent with $(S,y)$:
\begin{align} \label{eq:satvat_uv_consistent}
	\satvar_{uv}(S,y) = \satvar_{uv} \; \land \; \bigwedge_{x \in y^{-1}(0)} (x \notin H) \; \land \; \bigwedge_{x \in y^{-1}(1)}(x \in H) \enspace.
\end{align}

Algorithm~\ref{alg:polychecker} gives the pseudocode of \PolyChecker. The rest of this section describes the construction of $\satvar_{uv}$ and sketches the proof of Theorem~\ref{thm:polycheck}. All missing details can be found in Appendix~\ref{apx:polycheck}.
\begin{algorithm2e}
	\DontPrintSemicolon
	\caption{\PolyChecker}
	\label{alg:polychecker}
	\KwInput{$G=(V,E)$, and a labeled sample $(S,y)$}
	\KwOutput{$H \in \scHm(G)$ consistent with $(S,y)$, or failure if no such $H$ exists}
	\For{$uv \in E$}{
		compute the constraints given in Appendix~\ref{apx:polycheck}, and their conjunction $\satvar_{uv}$\;
		compute $\satvar_{uv}(S,y)$ as given by Equation~\eqref{eq:satvat_uv_consistent}\;%
		compute a solution $H$ to $\satvar_{uv}(S,y)$ and return it, if one exists
	}
	\textbf{return} ``$H$ does not exist''
\end{algorithm2e}

\vskip1em
\noindent\textbf{Constructing $\satvar_{uv}$.} To construct $\satvar_{uv}$, we start by computing the following subsets:
\begin{itemize}[itemsep=2pt,parsep=0pt]
	\item $\triangle^-_{uv}=N(u)\cap N(v)$. Note that $u,v \notin \triangle^-_{uv}$.
	\item $\square_{uv} = \{x \in V : x$ appears in a $4$-cycle having $uv$ as an edge$\}$. Note that $u,v \in \square_{uv}$.
	\item $A = \square_{uv} \setminus \triangle^-_{uv} = \{x \in \square_{uv} : d(x,u) \ne d(x,v) \}$.
	\item $A^u = \{x \in A : d(x,u)<d(x,v)\}$.
	\item $A^v = \{x \in A : d(x,u)>d(x,v)\} = A \setminus A^u$.
	\item $T = G \setminus E(G[\square_{uv} \cup \triangle^-_{uv}])$.
\end{itemize}
Note that all those sets are computable in time $\poly(|V|)$.  

\medskip
\noindent\textbf{A set of constraints.} As said, $\satvar_{uv}$ is a conjunction of several constraints (i.e., 2-SAT formulas). Each constraint is based on some of the sets defined above. We describe some of those constraints to convey the idea. The complete list is given in Appendix~\ref{apx:polycheck}, where the full proof of Theorem~\ref{thm:polycheck} checks that all of them are satisfied if and only if $H \in \scHm(G)$. 

The first constraint we present here, given by the formula $\satvar_{\conv}$, ensures that $A^u$ and $A^v$ are closed under monophonic convex hulls:
\begin{align}
	\satvar_{\conv} &= \bigwedge_{x \in \conv(A^u)} \!\!\!\! (x \in H) \quad \wedge \; \bigwedge_{x \in \conv(A^v)} \!\!\!\! (x \notin H) \enspace. \label{constr:conv}
\end{align}
Note that, as computing monophonic convex hulls in an $n$-vertex graph takes time $\poly(n)$~\citep{dourado2010complexity}, we can then compute $\satvar_{\conv}$ in time $\poly(n)$, too.

The next constraint, $\satvar_{\triangle^-_{uv}}$, ensures that $G[\triangle^-_{uv} \cap H] \subseteq G[\Border(H)]$ and $G[\triangle^-_{uv} \cap \overline H] \subseteq G[\Border(\overline H)]$ are both cliques, as required by Lemma~\ref{lem:H_charact}. In fact it ensures a slightly stronger condition: that the edge complement of $G[\triangle^-_{uv}]$ is a bipartite graph where adjacent vertices have opposite labels under $H$, as imposed by Lemma~\ref{lem:neg_bipartite}:
\begin{align}
    \satvar_{\triangle^-_{uv}} &= \bigwedge_{\{x,y\} \in E(\neg{G[\triangle^-_{uv}]})} (x \in H \vee y \in H) \wedge (x \notin H \vee y \notin H) \enspace.
\end{align}

Denote by $\CC(T)$ the set of connected components of $T$. It is not hard to prove that every $T_i \in \CC(T)$ is either the m-shadow $z/v$ or in the m-shadow $z/u$ for some $z \in \triangle_{uv}$. Therefore, either $T_i \subseteq H$ or $T_i \subseteq \overline H$. This is captured by the next constraint:
\begin{align}
	\satvar_{T} &= \bigwedge_{T_i \in \CC(T)} \; \bigwedge_{x,y \in V(T_i)} (x \in H \vee y \notin H) \wedge (x \notin H \vee y \in H) \enspace.
\end{align}

Finally, we compute constraints ensuring that there is no induced path starting from $A^u$ or $\conv(A_u)$ that violates monophonic convexity. To compute them, we first need to define what follows: for $X,Y \subseteq V$ and $k \in \{3,4\}$ denote by $\Pi_k(X,Z)$ the set of all induced paths on $k$ vertices that connect $X$ to $Y$. We report only one of such constraints them as an example:
\begin{align}
	\satvar_{u,3} &= \bigwedge_{\substack{\pi=(x,y,z)\\\pi \in \Pi_3(A^u,\triangle^-_{uv})}} (y \in H \vee z \notin H) \enspace.
\end{align}
It is easy to see that $\satvar_{\triangle^-_{uv}} $, $\satvar_{T}$, and $\satvar_{u,3}$ can be computed in time $\poly(n)$, too.

\section{Bounding the Number of Monophonic Halfspaces}\label{sec:fpt}
This section describes an algorithm to enumerate all m-halfspaces in FPT time and, through an additional charging argument, a near-tight bound on $|\scHm(G)|$.
The first intuition behind our algorithm comes from Lemma~\ref{lem:H_charact}, as follows. First, we guess each edge $uv$ of $G$ as a cut edge of some $H \in \scHm$. Starting from the endpoints $u,v$, we then try to find the vertices in $U=\Border(H)\cup\Border(\overline H)$. Since by Lemma~\ref{lem:H_charact} both $G[\Border(H)]$ and $G[\Border(\overline H)]$ are cliques, then $|U| \le 2\omega(G)$. At that point we exhaustively guess the partition of $U$ into $\Border(H)$ and $\Border(\overline H)$, and we check if that partition is indeed an m-halfspace. This would result in a total running time of $2^{2\omega(G)} \poly(n)$ per edge and, therefore, in total.

Unfortunately, this idea does not work straight away. The reason is that there is no efficient way to find all of $U$. However, by Lemma~\ref{lem:all_shadows}, we can show that this approach still works if instead of $U$ we look at the set $\triangle_{uv} = (N(u) \cap N(v)) \cup \{u,v\}$. Thus, the idea is to compute $\triangle_{uv}$, iterate over all possible labelings of it, and for each such labeling use Lemma~\ref{lem:all_shadows} to infer $H$ and check whether $H$ is a halfspace. This still does not yield the desired bound, as $\triangle_{uv}$ could be much larger than $\omega(G)$; in fact we could have $|\triangle_{uv}|=\Omega(n)$ even though $\omega(G)=\scO(1)$. Thus, we first check if $\triangle_{uv}$ is a subset of $U$, which can be done by just testing if $\neg G[\triangle_{uv}]$ is bipartite. If this is the case, then we already know $|\triangle_{uv}| \le 2\omega(G)$. By a careful analysis of $\neg G[\triangle_{uv}]$ we can then show that we need only to test certain labelings of $\triangle_{uv}$, whose number can be bounded by roughly $2^{\omega(G)}$.
The pseudocode of \AlgoList is given in Appendix~\ref{apx:fpt_list}, together with the proof of:
\begin{theorem}\label{thm:algo_listing}
$\AlgoList(G)$ lists $\scHm(G)$ in time $2^{\omega(G)} \poly(n)$.
\end{theorem}
The proof of Theorem~\ref{thm:algo_listing} immediately implies $|\scHm(G)| \le 2^{\omega(G)} 2m$.
However, starting from the analysis of \AlgoList, and using a charging argument that takes into account that an edge $uv$ can appear in the cut $\delta(H)$ of many m-halfspaces, we can prove the improved bound stated in Theorem~\ref{thm:halfspaces_count} (see Appendix~\ref{apx:halfspaces_count} for the complete proof).

\acks{MB and EE acknowledge the financial support from the FAIR (Future Artificial Intelligence Research) project, funded by the NextGenerationEU program within the PNRR-PE-AI scheme and the the EU Horizon CL4-2022-HUMAN-02 research and innovation action under grant agreement 101120237, project ELIAS (European Lighthouse of AI for Sustainability).
MT acknowledges support from a DOC fellowship of the Austrian academy of sciences (ÖAW) and partial support from FFG through the AI4SAR project (885324) in
the ASAP programme.
}

\bibliography{references}

\clearpage
\appendix
\section{Proofs of Auxiliary Lemmas}%
\label{apx:fpt}
In this section, we provide the proofs of some structural properties that are stated in the main body, and are crucial in proving the main results of this work.
We restate each lemma for convenience.

\allShadows*
{\renewcommand{\proofname}{Proof of Lemma~\ref{lem:all_shadows}.}
\begin{proof}
	We first show that $\bigcup_{z \in H \cap \triangle_{uv}} z/v\subseteq H$. This follows from the fact that for $z\in H$ and $v\in \overline H$ it holds $z/v\subseteq H$, as $H$ is a monophonic halfspace. 
	
	For the other direction, we have to show that for all $x\in H$ there exists a $z\in H\cap \triangle_{uv}$ such that $x\in z/v$. If $x\in u/v$ we are done. Also, $x\notin v/u$, as $v/u\subseteq \overline H$. Hence, we can assume that  $x\notin u/v\cup v/u$. Take an induced path $\pi_u$ from $u$ to $x$ and an induced path $\pi_v$ from $v$ to $x$. Denote the closest vertex to $u$ where $\pi_u$ and $\pi_v$ meet as $z$ ($z=x$ is possible). Let $C_{uvz}$ denote the cycle given by the subpaths going from $u$ and $v$ to $z$ and the edge $\{u,v\}$. Assume that $C_{uvz}$ is not an induced cycle. Then there is a chord $ab$ in $C_{uvz}$ going from $\pi_u$ to $\pi_v$ (or vice versa) and we could modify $\pi_u$ (respectively $\pi_v$) to meet $\pi_v$ already in $b$. We can iterate this process until we are left with an induced cycle $C_{uvz}$. Assume that $C_{uvz}$ consists of at least 5 vertices. Then either $u,v\in V(C_{uvz})\subseteq H$ or $u,v\in V(C_{uvz})\subseteq \overline H$ a contradiction to $u\in H$ and $v\in \overline H$. Also, $C_{uvz}$ cannot consist of exactly 4 vertices as in this case $x\in u/v \cup v/u$. Thus, $C_{uvz}$ is an induced triangle and $z\in (N(u)\cap N(v))\setminus\{u,v\}$. By construction $z$ is along an induced path from $u$ to $x$ and from $v$ to $x$ and hence $x\in z/v$ as required. 
\end{proof}
}

\edgeShadowsPoly*
{\renewcommand{\proofname}{Proof of Lemma~\ref{lem:shadows_in_P}.}
\begin{proof}
	Let $L = N(v) \setminus \{z\}$ and $S = \{x \in V \mid L \text{ is not a } (z,x)\text{-separator}\}$. It is straightforward that $S$ can be computed in polynomial time, hence we only need to show that $z/v=S$.
	
	Suppose first that $x \in S$. By definition of $S$ there exists a path $\pi$ between $x$ and $z$ such that $V(\pi) \cap L = \emptyset$. Thus $v$ is not adjacent to $V(\pi)$, and this holds obviously for any subpath $\pi'$ of $\pi$. Take in particular an induced subpath $\pi'$ of $\pi$ between $x$ and $z$, and observe that extending it to $v$ yields an induced path between $x$ and $v$ that contains $z$. Hence $x \in z/v$.
	
	Suppose now that $x \in z/v$. By definition of $z/v$ there exists an induced path $\pi$ between $x$ and $v$ that contains $z$. Clearly $\pi$ does not contain any $y \in N(v) \setminus \{z,v\}$, and thus its prefix $\pi'$ between $x$ and $z$ does not contain any $y \in N(v) \setminus \{z\} = L$. It follows that $L$ is not a $(z,x)$-separator.
\end{proof}
}

\borderAreCliques*
{\renewcommand{\proofname}{Proof of Lemma~\ref{lem:H_charact}.}
 \begin{proof}
         Let $H \in \scHm(G)$. Suppose $\{x,x'\} \notin E$ for some distinct $x,x' \in \Border(H)$. Choose $y \in N(x) \setminus H$ and $y' \in N(x') \setminus H$ that have smallest distance in $G[\overline H]$. (Note that $N(x) \setminus H \subseteq \overline H$ and $N(y') \setminus H \subseteq \overline H$, hence $y,y'$ exist). Let $\pi$ be a shortest path between $y$ and $y'$ in $G[\overline H]$. Observe that the concatenation of $x,\pi,x'$ is an induced path between $x$ and $x'$ that intersects $\overline H$. This contradicts $H \in \scHm(G)$. We conclude that $\{x,x'\} \in E$ for all distinct $x,x' \in \Border(H)$, hence $G[\Border(H)]$ is a clique. The same holds with $\overline H$ in place of $H$.

Now let $H \subseteq V$ and suppose $G[\Border(H)]$ and $G[\Border(\overline H)]$ are cliques; we show that both $H$ and $\overline H$ are monophonically convex. Suppose for contradiction that an induced path between distinct $x,x' \in H$ intersects $\overline H$. Without loss of generality we may assume $x,x'$ are the endpoints of the path, and thus by definition in $\Border(H)$ by definition. This implies $\{x,x'\} \in E$, contradicting the fact that the path is induced. The same holds with $\overline H$ in place of $H$.
 \end{proof}
 }

\cliqueConnectedComps*
{\renewcommand{\proofname}{Proof of Lemma~\ref{lem:clique_connectedcomps}.}
\begin{proof}
 Observe that
$
	N_{uv} \le \alpha(\neg G[\triangle_{uv}]) = \omega(G[\triangle_{uv}]) \le \omega(G)
$.
\end{proof}
}

\section{Missing Details for Section~\ref{sec:polycheck}}
\label{apx:polycheck}

\subsection{Full list of constraints for $\satvar_{uv}$}
\begin{align}
	\satvar_{\triangle^-_{uv}} &= \bigwedge_{\{x,y\} \in E(\neg{G[\triangle^-_{uv}]})} (x \in H \vee y \in H) \wedge (x \notin H \vee y \notin H) \\[8pt]
	\satvar_{\conv} &= \bigwedge_{x \in \conv(A^u)} \!\!\!\! (x \in H) \quad \wedge \; \bigwedge_{x \in \conv(A^v)} \!\!\!\! (x \notin H)  \label{apx:constr:conv} \\[8pt]
\satvar_{T} &= \bigwedge_{T_i \in \CC(T)} \; \bigwedge_{x,y \in V(T_i)} (x \in H \vee y \notin H) \wedge (x \notin H \vee y \in H) \\[8pt]
    \satvar_{u,3} &= \bigwedge_{\substack{\pi=(x,y,z)\\\pi \in \Pi_3(A^u,\triangle^-_{uv})}} (y \in H \vee z \notin H)\;, \qquad
    \satvar_{v,3} = \bigwedge_{\substack{\pi=(x,y,z)\\\pi \in \Pi_3(A^v,\triangle^-_{uv})}} (y \notin H \vee z \in H) \label{constr:u3}\\[8pt]
	\satvar_{u,4} &= \bigwedge_{\substack{\pi=(x,w,y,z)\\\pi \in \Pi_4(A^u,\triangle^-_{uv})}} (y \in H\vee z\notin H)\;, \qquad
	\satvar_{v,4} = \bigwedge_{\substack{\pi=(x,w,y,z)\\\pi \in \Pi_4(A^v,\triangle^-_{uv})}} (y\notin H \vee z\in H) \\[8pt]
    \satvar_{A^u} &= \bigwedge_{\substack{\{x,y\} \in E\\
    x\in\conv(A^u)}} \bigwedge_{z \in \triangle^-_{uv} \setminus N(x)} \!\!\!(y \in H \vee z \notin H)\;, \quad
    \satvar_{A^v} = \bigwedge_{\substack{\{x,y\} \in E\\ x\in\conv(A^v)}} \bigwedge_{z \in \triangle^-_{uv} \setminus N(x)} \!\!\!(y \notin H \vee z \in H)  \label{constr:Au}
\end{align}

\subsection{Proof of Theorem~\ref{thm:polycheck}}
Before proving Theorem~\ref{thm:polycheck} we need another technical lemma.
\begin{lemma}\label{lem:sq_cup_tr}\label{lem:square_labels}\label{lem:ext_cc}
	Let $H \in \scHm(G)$ and $uv \in \Cut(H)$.  Then, (i) $A^u \subseteq H$ and $A^v \subseteq \overline H$, and (ii) $T_i \subseteq H$ or $T_i \subseteq \overline H$, for all $T_i \in \CC(T)$.
\end{lemma}
\begin{proof}
	(i) Let $x \in A^u$ and suppose by contradiction $x \notin H$. In particular $x \ne u$, so $d(x,u) \ge 1$. Now $(x,u,v)$ is an induced path with endpoints in $H$ containing a vertex of $\overline{H}$, contradicting $H \in \scHm(G)$. A symmetric argument applies to $A^v$.
	
	(ii) First, let us prove that $\Gamma(H) \cup \Gamma(\overline H) \subseteq \square_{uv} \cup \triangle^-_{uv}$. Let $x \in \Border(H)$. If $x=u$ then $x \in \square_{uv}$, so we may assume $x \in \Border(H) \setminus \{u\}$, which by Lemma~\ref{lem:H_charact} implies $x \in N(u)$. Consider then any edge $xy \in \Cut(H)$. If $y=v$ then $u \in \triangle^-_{uv}$. If instead $y \ne v$, then $y \in N(v)$ by Lemma~\ref{lem:H_charact}; so $(u,x,y,v)$ is a $4$-cycle, and $x,y \in \square_{uv}$.%
	We conclude that $\Gamma(H) \cup \Gamma(\overline H) \subseteq \square_{uv} \cup \triangle^-_{uv}$. This implies that $T$ is an edge-subgraph of $R = G \setminus E(G[\Gamma(H) \cup \Gamma(\overline H)])$. Therefore, every connected component $T_i \in \CC(T)$ is entirely contained in some connected component $R_i \in \CC(R)$. Thus, it is sufficient to prove the claim for $R_i$. Suppose for contradiction that $R_i$ intersects both $H$ and $\overline H$. Since $R_i$ is connected there is an edge $\{x,y\} \in E(R_i)$ with $x \in H$ and $y \in \overline H$, so $\{x,y\} \in E(R_i) \cap \Cut(H)$. But $\Cut(H) \subseteq E(G[\Gamma(H) \cup \Gamma(\overline H)])$ and thus $E(R) \cap \Cut(H) = \emptyset$, a contradiction.
\end{proof}

\polyconsistency*
{\renewcommand{\proofname}{Proof of Theorem~\ref{thm:polycheck}.}
\begin{proof}
    First of all, note that $\satvar_{uv}$ is a 2-SAT instance, and that it can be computed in polynomial time (see above). The same is then obviously true for $\satvar_{uv}(S_-,S_+)$, too. It remains to prove that $H \in \scHm(G)$ consistent with $(S_-,S_+)$ exists if and only if $\satvar_{uv}(S_-,S_+)$ is satisfied for some $uv \in E$. We prove the two directions separately.
    
    \medskip
    \noindent \textbf{Forward direction.} Suppose $H \in \scHm(G)$ is consistent with $(S_-,S_+)$. We check~\eqref{apx:constr:conv}--\eqref{constr:Au} and show $H$ satisfies $\satvar_{uv}$ for every $uv \in \Cut(H)$. This implies that $H$ satisfies $\satvar_{uv}(S_-,S_+)$ for every $uv \in \Cut(H)$ and, thus, for some $uv \in E$.
    \begin{itemize}[itemsep=-2pt]
        \item $\satvar_{\conv}$. First, $A^u \subseteq H$ and $A^v \subseteq \overline H$ by Lemma~\ref{lem:square_labels}. Since $H$ is closed under $\conv$ this implies $\conv(A^u) \subseteq H$ and $\conv(A^v) \subseteq \overline H$, too.
        \item $\satvar_{\triangle^-_{uv}}$. Note that $\triangle^-_{uv} \subseteq \Border(H) \cup \Border(\overline H)$ and apply Lemma~\ref{lem:neg_bipartite}.
        \item $\satvar_T$. Apply Lemma~\ref{lem:ext_cc}.
        \item $\satvar_{u,3}$ and $\satvar_{v,3}$. Suppose by contradiction $y \notin H$ and $z \in H$. Since $(x,y,z)$ is induced, this violates the convexity of $H$. Thus, $\satvar_{u,3}$ is satisfied. A symmetric argument applies to $\satvar_{v,3}$.
        \item $\satvar_{u,4}$. The arguments are similar to those for $\satvar_{u,3}$ and $\satvar_{v,3}$.
        \item $\satvar_{A^u}$ and $\satvar_{A^v}$. Suppose the formula fails, so there exists $\{x,y\} \in E$ with $x \in \conv(A^u)$ and $y \notin H$, as well as $z \in \triangle^-_{uv} \setminus N(x)$ with $z \in H$. As $x \in H, y \notin H$, and $\{x,y\}\in E$, then $\{x,y\} \in \Cut(H)$. As $z \in H, v \notin H$, and $z \in \triangle^-_{uv}$, then $z  \in \Cut(H)$, too. But then by Lemma~\ref{lem:H_charact} we have $\{x,y\} \in E$, a contradiction. A symmetric argument applies to $\satvar_{A^v}$.
    \end{itemize}
    
    \medskip
    \noindent \textbf{Backward direction.}  Suppose $H \subseteq V$ satisfies $\satvar_{uv}(S_-,S_+)$ for some $uv \in E$. This implies that $H$ satisfies $\satvar_{uv}$, which in turn implies $uv \in \Cut(H)$ by the constraints of~\eqref{apx:constr:conv}. We now show that $H \in \scHm(G)$. Suppose indeed by contradiction $H \notin \scHm(G)$. Note how this implies $\conv(H) \ne H$ or $\conv(\overline H) \ne \overline{H}$. Without loss of generality we assume $\conv(H) \ne H$. Therefore there exists an induced path $P = G[\{x_1,\ldots,x_k\}]$ in $G$ such that $x_1,x_k \in H$ and $x_i \in \overline H$ for all $i=2,\ldots,k-1$. We will use the path $P$ to prove a contradiction.
    
    First, since $\satvar_T$ is satisfied, then $\{x_1,x_2\} \in E(G[\square_{uv} \cup \triangle^-_{uv}]) = E(G[A \cup \triangle^-_{uv}])$. By the same argument $\{x_{k-1},x_k\} \in E(G[A \cup \triangle^-_{uv}])$. Moreover, since $\satvar_{\conv}$ is satisfied then $\{x_1,x_k\} \cap A^v = \emptyset$ and $\{x_2,x_{k-1}\} \cap A^u = \emptyset$; hence, $x_1,x_k \in A^u \cup \triangle^-_{uv}$ and $x_2,x_{k-1} \in A^v \cup \triangle^-_{uv}$.
    
    Now, if $x_1,x_k \in A^u$, then $V(P) \subseteq H$ since $\satvar_{\conv}$ is satisfied; hence $|\{x_1,x_k\} \cap A^u| \le 1$. If instead $x_1,x_k \in \triangle^-_{uv}$, then since $\satvar_{\triangle^-_{uv}}$ is satisfied $\{x_1,x_k\} \notin E(\neg G[\triangle^-_{uv}])$, implying the absurd $\{x_1,x_k\} \in E(G)$; hence $|\{x_1,x_k\} \cap \triangle^-_{uv}| \le 1$.
    We conclude that precisely one of $x_1,x_k$ is in $A^u$ and the other one in $\triangle^-_{uv}$.
    Without loss of generality say $x_1 \in A^u$ and $x_k \in \triangle^-_{uv}$. Then $k \ge 5$, for otherwise $\pi=P$ would falsify $\satvar_{u,3}$ or $\satvar_{u,4}$. As $k \ge 5$ and $P$ is induced we deduce $x_2,x_{k-1}$ are distinct and $\{x_1,x_{k-1}\},\{x_2,x_{k-1}\} \notin E$.
    
    Next we show that $x_2 \in \triangle^-_{uv}$ and $x_{k-1} \in A^v$. Suppose by contradiction $x_2 \in A^v$. Letting $x,y=x_1,x_2$ and $z=x_k$, then, $\satvar_{A^u}$ fails. Thus $x_2 \in \triangle^-_{uv}$.
    Now suppose by contradiction $x_{k-1} \in \triangle^-_{uv}$. Then $\{x_2,x_{k-1}\} \in E(\neg G[\triangle^-_{uv}])$, and since $\satvar_{\triangle^-_{uv}}$ is satisfied then $x_2 \in H$ or $x_{k-1} \in H$, a contradiction. Thus $x_{k-1} \in A^v$.
    Now let $x,y=x_{k-1},x_k$ and $z=x_2$. Then $\satvar_{A^v}$ fails. We conclude that $P$ does not exist. Hence $H = \conv(H)$, and $H \in \scHm(G)$. %
\end{proof}
}

\subsection{Proof of Theorem~\ref{thm:polydelayversionspace}} \label{apx:polydelay}
\polydelayversionspace*
The pseudocode of \PolyDelay is given in Algorithm~\ref{alg:polydelaylist}. It is straightforward to see that the algorithm constructs a binary search tree whose leaves are precisely the elements of the version space $\VS((S,y),\scHm(G))$, which are the only ones that are output. Moreover, by Theorem~\ref{thm:polycheck} \PolyDelay spends time $\poly$ for every node in the search tree. The claim of Theorem~\ref{thm:polydelayversionspace} follows.
\begin{algorithm2e}
\DontPrintSemicolon
\caption{\PolyDelay}
\label{alg:polydelaylist}
\KwInput{$(S,y),G=(V,E)$}
\KwOutput{the version space $\VS((S,y),\scHm(G))$}
let $H^+=y^{-1}(1)$ and $H^-=y^{-1}(0)$\;
\uIf{$H^+\cup H^-=V$}{
  \textbf{print} $H^+$ and \textbf{return}\;
}
pick any $x \in V \setminus (H^+ \cup H^-)$\;
\lIf{ $\PolyChecker(G, H^+ \cup \{x\}, H^-)$}{
    $\PolyDelay(G,H^+ \cup \{x\},H^-)$
}
\lIf{ $\PolyChecker(G, H^+, H^- \cup \{x\})$}{
    $\PolyDelay(G,H^+,H^- \cup \{x\})$
}
\end{algorithm2e}

\section{\AlgoList, and a proof of Theorem~\ref{thm:algo_listing}} \label{apx:fpt_list}

\begin{algorithm2e}[h]
	\DontPrintSemicolon
	\caption{\AlgoList}
	\label{alg:fpt_enum}
	\KwInput{$G=(V,E)$}
	\KwOutput{$\scHm(G)$}
	let $B=\emptyset$\\
	\textbf{print} $\emptyset, V$\\
	\For{$\{u,v\}\in E$\label{line:fpt_for_uv}}{
		let $\triangle_{uv} = (N(u)\cap N(v)) \cup \{u,v\}$\\
		\If{$\neg{G[\triangle_{uv}]}$ is bipartite\label{line:fpt_check_bipartite}}{
			let $R \subseteq \triangle_{uv}$ contain one vertex $x_C$ from each connected component $C$ of $\neg{G[\triangle_{uv}]}$\\
			let $S = \emptyset$\\
			\For{each subset $U \subseteq R$\label{line:fpt_for_U}}{
				\For{each connected component $C$ of $\neg{G[\triangle_{uv}]}$\label{line:fpt_for_C}}{
					\leIf{$x_C \in U$}{$S = S \cup C(x_C)$}{$S = S \cup V(C) \setminus C(x_C)$}
				}
				let $X = \bigcup_{z \in S} z/v$\\
				\lIf{$X$ and $V \setminus X$ are m-convex and $\Border(X)\cap B = \emptyset$\label{line:fpt_if_mp}}{\textbf{print} $X$}
			}
			$B = B\cup \{\{u,v\}\}$
		}
	}
\end{algorithm2e}

Before proving the theorem we need one auxiliary lemma.

\begin{lemma}\label{lem:neg_bipartite}
Let $H \in \scHm(G)$ and let $F = \neg {G[\Border(H) \cup \Border(\overline H)]}$. Then, every connected component $\hat F$ of $F$ is bipartite with sides $V(\hat F) \cap H$ and $V(\hat F) \cap \overline H$.
\end{lemma}
\begin{proof}
	The first claims follows immediately from Lemma~\ref{lem:H_charact}. For the second claim observe that if $\{x,y\} \in E(F)$ then $\{x,y\} \notin E(G)$ and thus, again by Lemma~\ref{lem:H_charact}, $\{x,y\} \not\subseteq H$ and $\{x,y\} \not\subseteq \overline H$; therefore, precisely one of $x$ and $y$ is in $H$.
\end{proof}

\noindent\textbf{Proof of Theorem~\ref{thm:algo_listing}.}
For the correctness, the condition at line~\ref{line:fpt_if_mp} ensures \AlgoList\ lists only elements of $\scHm(G)$; thus we need just to prove that every $H \in \scHm(G)$ is listed exactly once. Let $\{u,v\}$ be the first edge of $\Cut(H)$ considered at line~\ref{line:fpt_for_uv} (recall that we assume $G$ to be connected, hence such an edge exists). Note that every $x \in N(u) \cap N(v)$ is incident with an edge in $\Cut(H)$ (the other endpoint being $u$ if $x \in \overline H$ and $v$ if $x \in H$), and the same holds obviously for $x \in \{u,v\}$. Therefore $\triangle_{uv} \subseteq \Border(H) \cup \Border(\overline H)$. Lemma~\ref{lem:neg_bipartite} then implies $\neg {G[\triangle_{uv}]}$ is bipartite, so the condition at line~\ref{line:fpt_check_bipartite} holds. Next, consider the loop at line~\ref{line:fpt_for_U}. Clearly, at some iteration $U = R \cap H$. At that point, again by Lemma~\ref{lem:neg_bipartite} the loop at line~\ref{line:fpt_for_C} ensures $S = \triangle_{uv} \cap H$. By Lemma~\ref{lem:all_shadows}, then, $X = \bigcup_{z \in S} z/v = H$. The condition at line~\ref{line:fpt_if_mp} is then satisfied, thus $H$ is printed; moreover in every subsequent iteration of the loop of line~\ref{line:fpt_for_uv} $\{u,v\} \in B$ and thus $H$ will not be printed.

For the running time, note that the loop at line~\ref{line:fpt_for_U} is executed only if $\neg G[\triangle_{uv}]$ is bipartite, in which case the loop performs $2^{N_{uv}}$ iterations where $N_{uv}$ is the number of connected components of $\neg G[\triangle_{uv}]$. By Lemma~\ref{lem:clique_connectedcomps} $N_{uv}\leq\omega(G)$.
Hence the loop at line~\ref{line:fpt_for_uv} makes at most $2^{\omega(G)}$ iterations. To conclude the proof note that every other step performed by \AlgoList\ (including computing $X$, see Lemma~\ref{lem:shadows_in_P}) takes time $\poly(n)$.

\subsection{Proof of Theorem~\ref{thm:halfspaces_count}}
\label{apx:halfspaces_count}
\halfspaceCount*
\begin{proof}
    For any $e \in E$, let $\scHm(e) = \{H \in \scHm(G) \mid e \in \Cut(H)\}$. Recall that each $x \in \Border(H)\cup\Border(\overline H)$ is incident with some edge in $\Cut(H)$ by definition, and therefore
	\begin{align} \label{eq:hcount_prop1}
		|\Cut(H)| \ge \max\{|\Border(H)|,|\Border(\overline H)|\} \ge \frac{|\Border(H)\cup\Border(\overline H)|}{2} \enspace.
	\end{align}
	Moreover, for any $e \in E$ let $\omega(e) = \omega(G[\triangle_{uv}])$ where $e=\{u,v\}$, see Algorithm~\ref{alg:fpt_enum} and the proof of Theorem~\ref{thm:algo_listing}. Note that, since $\omega(G[\triangle_{uv}]) \ge \frac{|\triangle_{uv}|}{2}$, then
	\begin{align} \label{eq:hcount_prop2}
		\frac{2^{\omega(e)}}{|\triangle_{uv}|} = \frac{2^{\omega(G[\triangle_{uv}])}}{|\triangle_{uv}|} \le \frac{2^{1+\omega(G[\triangle_{uv}])}}{\omega(G[\triangle_{uv}])} \le \frac{2^{1+\omega(G)}}{\omega(G)} \enspace,
	\end{align}
	where the last inequality follows by monotonicity of $2^x/x$ for $x \ge 2$ and since $\omega(G[\triangle_{uv}])\ge 2$ as $\{u,v\} \subseteq \triangle_{uv}$.
	We then obtain (not counting the two trivial halfspaces $\emptyset, V$):
	\begin{align}
		|\scHm(G)| -2 &= \sum_{H \in \scHm(G)} \sum_{e \in \Cut(H)} \frac{1}{|\Cut(H)|} \\
		&= \sum_{e \in E} \sum_{H \in \scHm(e)} \frac{1}{|\Cut(H)|} && \text{by rearranging terms} \\
		& \le \sum_{e \in E} \sum_{H \in \scHm(e)} \frac{2}{|\Border(H)\cup\Border(\overline H)|} && \text{by \eqref{eq:hcount_prop1}}\\
		&\le \sum_{e \in E} |\scHm(e)| \frac{2}{|\triangle_{e}|} && \text{as $\triangle_{e} \subseteq \Border(H)\cup\Border(\overline H)$}\\
		&\le \sum_{e \in E} 2^{\omega(e)} \frac{2}{|\triangle_{e}|} && \text{see proof of Theorem~\ref{thm:algo_listing}}\\
		& \le \sum_{e \in E} \frac{2^{2+\omega(G)}}{\omega(G)} && \text{by \eqref{eq:hcount_prop2}}\\
		&= 4 m \frac{2^{\omega(G)}}{\omega(G)} \enspace.
	\end{align}
	The proof is complete.
\end{proof}

\section{Supervised Learning}
We study the standard PAC (probably approximately correct) setting with a given finite instance space $V$ and a known hypothesis space $\scH\subseteq 2^V$. 
In our case $\scH=\scHm(G)$ is defined implicitly by the input graph $G$ given to the algorithm.
The standard approach to PAC learning is to return a hypothesis, a m-halfspace in our case, that is as close as possible to being consistent with the given labeled sample $(S,y)$.
In particular, we say that it is possible to $(\varepsilon, \delta)$-PAC learn $(V, \scHm)$ if we can provide a m-halfspace $H=H(S,y)$ that has error bounded by $\varepsilon$ with probability at least $1-\delta$, for any $\varepsilon,\delta\in(0,1)$, under some unknown joint distribution over vertices and their labels.
The aim of PAC learning is to guarantee this property with the smallest sample complexity possible.

While standard PAC learning formulations typically measure the running time as a function of the sample size (see, e.g., \citet{valiant1984theory,shalev2014understanding}), we measure running time as a function of $n$. Note that a dependence on $n$ is unavoidable: even just to check whether two vertices $u,v \in V$ are separated by some $H \in \scHm(G)$, one needs to distinguish between $G$ being a path or a cycle, which takes time $\Omega(n)$.

We start with a bound on the VC dimension of $\scHm(G)$, thus determining learnability in the supervised setting. We write $\VC(V,\scH)$, or simply $\VC(\scH)$, for the VC dimension of the hypothesis space $\scH\subseteq 2^V$.
\monovc*
\begin{proof}
We first provide a proof for the upper bound by showing that the chain of inequalities $\VC(V,\scHm(G)) \le r(G) \le \Tilde{\omega}(G)$ holds. The first inequality holds for halfspaces in arbitrary convexity spaces, as any set shattered by halfspaces must be Radon independent \citep{moran2020weak}. The second inequality is a classic result in monophonic graph convexity theory \citep{duchet1988convex}.
Now, we show that the lower bound holds too. Indeed, if $G \in S_4$ then every clique in $G$ is shattered by $\scHm(G)$, since (i) it is Radon independent, and (ii) any of its partitions can be extended to a halfspace \citep{thiessen2021active}. Hence, $\VC(V,\scHm(G))\geq\omega(G)$.
\end{proof}
We now move on to the study of the sample complexity for PAC learning $(V,\scHm)$ in both the realizable and the agnostic settings.

\paragraph{Realizable case.} 
Since we are in the realizable setting, we know there exists some m-halfspace $H^*\in\scHm(G)$ that is consistent with $(S,y)$. While it is known that computing m-convex $p$-partitions of a graph for $p\ge 3$ is NP-hard \citep{gonzalez2020covering}, the complexity of computing m-halfspaces ($p=2$) is not known. A first step towards closing this gap is to use our listing algorithm to obtain an FPT algorithm for this problem. We go beyond that and fully close this gap by providing, in Theorem~\ref{thm:polycheck}, a polynomial-time algorithm relying on a non-trivial reduction to the 2-SAT problem.

By standard PAC results \citep{blumer1989learnability} we obtain the following near-optimal sample complexity bound, save possibly for the $\log(1/\varepsilon)$ factor.
\realizablePACpoly*
\begin{proof}
    \citet{blumer1989learnability} give a sample complexity of $\scO\Bigl(\frac{\VC(V,\scHm)\log(1/\varepsilon)+\log(1/\delta)}{\varepsilon}\Bigr)$ and the stated bound follows from the fact that $\VC(V,\scHm)\leq \omgtil(G)$ (Proposition~\ref{pro:VC_mono}). The runtime follows  by Theorem~\ref{thm:polycheck}.
\end{proof}

\paragraph{Agnostic case.} 
In the agnostic case we are given a sample $(S,y)$ that is not necessarily realizable; that is, there may be no $H^* \in \scHm(G)$ that is consistent with $(S,y)$. The classic goal is to compute an ERM (empirical risk minimizer), i.e., some $H_{\mathrm{ERM}} \in \scHm$ that minimizes the empirical risk $L(H,(S,y)) = \frac1{|S|} \sum_{s \in S} \mathbbm{1}\{y(s) \ne \mathbbm{1}_{s\in H}\}$ over all $H\in\scHm$. As a further corollary of our results above we obtain:
\ermfpt*
\begin{proof}
    We list all m-halfspaces using \PolyDelay\ and return any of the listed m-halfspaces that minimizes the number of mistakes on $S$.
\end{proof}
Similarly to the realizable case, we get the following agnostic PAC bound.
\agnosticPACFPT*
\begin{proof}
It is well known that an ERM yields the optimal agnostic sample complexity of the order $\scO\Bigl(\frac{\VC(V,\scHm)+\log(1/\delta)}{\varepsilon^2}\Bigr)$ \citep{vapnik1974theory, talagrand1994sharper}.
    The sample complexity follows by $\VC(V,\scHm)\leq \omgtil(G)$ (Proposition~\ref{pro:VC_mono}). We can compute an ERM in time $|\scHm|\poly(n)\leq 2^{\omega(G)}\poly(n)$ by Corollary~\ref{cor:erm-fpt}.
\end{proof}

\section{Active Learning}\label{apx:active}
This section presents an efficient algorithm for the problem of actively learning monophonic halfspaces, defined as follows. The algorithm is given a graph $G=(V,E)$, and nature selects a concept $H\in\scHm(G)$. The algorithm can query any vertex $x \in V$ for its label $\mathbbm{1}_{x \in H}$. The goal of the algorithm is to output $H$ by making as few queries as possible.
This problem is a special case of realizable transductive active learning on a graph \citep{afshani2007complexity, guillory2009label,cesa2010active,dasarathy2015s2}, and can be seen as a variant of query learning \citep{angluin1988queries, hegedHus1995generalized} with an additional fixed graph given where vertices correspond to the instance space.
Now let $\scH\subseteq 2^V$ be a concept class. The \emph{query complexity} of $(V,\scH)$ is the maximum number of queries an optimal algorithm makes over $H \in \scH$. More precisely, for any algorithm $A$ and any $H \in \scH$, let $\qc(A,H)$ denote the number of queries $A$ make on $G$ when nature chooses $H$. The query complexity of $A$ on $\scH$ is $\qc(A,\scH)=\max_{H\in\scH} \qc(A,H)$. The query complexity of $\scH$ is $\qc(\scH)=\min_A \qc(A,\scH)$.

Relying on Lemma~\ref{lem:all_shadows} we get the main result in this section: a polynomial-time algorithm achieving a graph-dependent near-optimal query complexity.
\polyactive*
{\renewcommand{\proofname}{Proof of Theorem~\ref{thm:active_upper_bound}.}
\begin{proof}
We compute a minimum monophonic hull set $S$ of size $h(G)$ in polynomial time using the algorithm of \citet{dourado2010complexity}. We query all vertices of $S$. If they have all the same label, all vertices in the graph have the same label, by the definition of hull sets. Otherwise, we take two vertices with different labels and find a shortest path, which is induced, in particular, between them. As each label class corresponds to a halfspace, the path has exactly one cut edge $e=\{u,v\}$. We identify $e$ using $\scO(\log\diam_g(G))$ queries through binary search on the path. Using the labels of $u$ and $v$ we can determine the labels of $u/v$ (same as $u$) and $v/u$ (same as $v$). We compute these sets efficiently  using Lemma~\ref{lem:shadows_in_P}. We only have to query one vertex in each connected component of $\neg G[\Delta_{uv}]$, which are at most $\omega(G)$ by Lemma~\ref{lem:clique_connectedcomps}. By Lemma~\ref{lem:all_shadows} the labels of the remaining vertices can be determined through the m-shadows of vertices in $\Delta_{uv}$ and $u$ or $v$.
\end{proof}
}
Along the previously mentioned separation axioms from Definition~\ref{def:sep_axioms} we achieve increasingly tighter lower bounds on the query complexity, eventually matching our algorithmic upper bound from Theorem~\ref{thm:active_upper_bound} for all $S_4$ graphs, the strongest separability assumption. 
\activeLowerBounds*
{\renewcommand{\proofname}{Proof of Theorem~\ref{thm:lower_bounds}.}
\begin{proof} For $S_2$ graphs note that any edge on a fixed induced path $P$ can be a cut edge of a m-halfspace. Locating this cut edge takes $\Omega(\log|V(P)|)$ queries in the worst case. For $S_3$ graphs take any minimum hull set $S$ (of size $h(G)$). Having queried only a proper subset $S'\subsetneq S$, the algorithm cannot decide between the halfspace $V$ and a halfspace separating $\conv(S')$ and some $s\in S\setminus S'$. Note that by minimality of $S$, the vertex $s\notin \conv(S')$.
For $S_4$ graphs any clique can be shattered by halfspaces. Hence, we fix a clique of maximum size and force the learner to query it completely.
\end{proof}
}

\section{Online Learning} \label{apx:online}
In this section, we devise near-optimal efficient algorithms for the problem of online learning of monophonic halfspaces. The classical realizable online learning problem of \citet{littlestone1988learning} (while the agnostic/unrealizable setting has been studied by, e.g., \citet{ben2009agnostic}) can be modelled as an iterative game between our learner and the environment over a finite number $T$ of rounds. The instance space $V$ and a hypothesis space $\scH\subseteq 2^V$ is known and fixed. First, the environment chooses a hypothesis $h$ from a hypothesis space $\scH$. Then, in each round $t=1,\dots,T$:
\begin{enumerate}[itemsep=2pt,parsep=0pt,topsep=4pt]
    \item the environment chooses a point $v_t\in V$,
    \item the learner predicts the label $\hat{y}_t\in\{0,1\}$ of $v_t$,
    \item the environment reveals the true label $y_t=\mathbbm{1}_{v_t\in h}$,
    \item the learner made a mistake if $\hat{y}_t \neq y_t$.
\end{enumerate}
The goal of the learner is to minimize the total number of mistakes. More precisely let $A$ be an algorithm for this online learning problem. Then, let $M(A,h)$ for $h\in\scH$ denote the worst-case number of mistakes $A$ would make on any sequence labeled by $h$.
The mistake bound of $A$ on $\scH$ is thus defined as $M(A,\scH)=\max_{h\in\scH} M(A,h)$ and we are interested in the optimal mistake bound $M(\scH)=\min_A M(A,\scH)$, also known as the Littlestone dimension of $\scH$. 

The node classification (or node labeling) variant of this problem is well studied \citep{herbster2005online,cesa2013random,herbster2015online}. As in the active learning variant the main parameter is the (potentially effective resistance weighted) cutsize, linearly determining the mistake bounds.
In this section, we mainly study the variant of the above realizable online learning problem over our hypothesis class $\scHm$.

We prove the main result in this section (Theorem~\ref{thm:polyAndFPTonline}) below, discussing the strategy using Halving and Winnow individually.

\subsection{Halving over Monophonic Halfspaces}
The Halving algorithm is a classical algorithm for online binary classification over a finite hypothesis class, used by \citet{littlestone1988learning} for providing mistake bounds, which implements an ``iterative halving'' procedure.
The main idea behind Halving is to maintain a set of hypotheses from the original class $\scH$ that are consistent with the labels observed so far; these consistent hypotheses form the so called \emph{version space}.
The version space of Halving begins with the entire class $\scH$ and it gradually shrinks while guaranteeing that the target hypothesis fixed by the environment belongs to the version space (by consistency).
The prediction of Halving at each round consists of a majority vote for the current point over the entire version space.

Our realizable online learning problem perfectly fits the requirements for running Halving, especially considering the fact that we can construct the nontrivial hypothesis class $\scHm(G)$ by running \AlgoList\ (Algorithm~\ref{alg:fpt_enum}).
Therefore, we can derive the following guarantees as a consequence of well-known results on Halving (e.g., \citet[Theorem~1]{littlestone1988learning}).

\begin{proposition}\label{thm:halving}
    Halving applied to $(V,\scHm(G))$ achieves a mistake bound of $\scO\Bigl(\omega(G) + \log \frac{n}{w(G)}\Bigr)$ and can be run in time $2^{\omega(G)}\poly(n)$ per round.
\end{proposition}
A straightforward implementation of Halving requires the enumeration of the version space in each round. By Theorem~\ref{thm:halfspaces_count}, we know that $|\scHm|\leq \frac{4m 2^{\omega(G)}}{\omega(G)}+2$ and that this inequality can be almost tight. Hence, if we run Halving by enumerating the version space we cannot hope for a better runtime.
Halving's mistake bound is given by the fact that it corresponds to $\scO(\log|\scHm|)$.

\subsection{Winnow over Monophonic Halfspaces} \label{apx:online_winnow}
The main downside of running the Halving algorithm is its running time.
In general we cannot hope to improve upon it unless we can manage to further restrict the initial size of the hypothesis class.
Nonetheless, if the target hypothesis can be represented as a sparse disjunction, the Winnow algorithm \citep{littlestone1988learning} can achieve a similar mistake bound as Halving without the need to enumerate the version space. A similar usage of Winnow for online node classification problems was discussed by \citet{gentile2013online}.
\begin{proposition}[\citet{littlestone1988learning}]
    Winnow achieves a mistake bound of $\scO(k\log d)$ in $\scO(\poly(kd))$ time to online learn monotone $k$-literal disjunctions on $d$ variables.
\end{proposition}
By leveraging on this property of the Winnow algorithm, we can show the following result.
\begin{theorem}
    Online learning of monophonic halfspaces is possible in time $\poly(n)$ with a mistake bound of $\scO(\omega(G)\log n)$.
\end{theorem}
\begin{proof}
    Let $E=\{\{u_1,v_1\},\dots,\{u_m,v_m\}\}$ and $S = (u_1/v_1,\dots,u_m/v_m,v_1/u_1,\dots,v_m/u_m)$. Each $S_i=u_i/v_i$ can be computed in polynomial time by computing the distance matrix once (all pairs shortest path). We represent each vertex $v\in V$ as a binary vector $\phi(v)\in\{0,1\}^{2m}$ (where, we recall, $m=|E(G)|$) such that $\phi(v)_i = \mathbbm{1}_{\{v\in S_i\}}$. 
    Using Lemma~\ref{lem:atmost_clique_cutedges}, we know that any monophonic halfspace can be represented as the union of at most $k\leq\omega(G)$ sets in the sequence $S$. Hence, any monophonic halfspace can be represented as a disjunction 
    \begin{equation}
        \phi(\cdot)_{i_1}\vee\dots\vee \phi(\cdot)_{i_k}
    \end{equation}
    of at most $k\leq\omega(G)$ of the $2m$ variables in $\phi(\cdot)$. Thus, we can run Winnow on $\phi(V)=\{\phi(v) \mid v\in V\}$ and achieve a mistake bound of $\scO(\omega(G)\log n)$ where we note that $\log(2m)=\scO(\log n)$. As each step of Winnow takes time $\scO(m)=\scO(n^2)$, the overall runtime is dominated by the construction of $\phi(V)$.
\end{proof}
Winnow nearly achieves the mistake bound as Halving in Theorem~\ref{thm:halving} (worse by only a logarithmic factor), yet in polynomial time without the need to enumerate the version space. The achieved mistakes is near-optimal. %
In particular we have the following simple result.
\mistakelowerbound*
{\renewcommand{\proofname}{Proof of Proposition~\ref{prop:online_lower_bound}.}
\begin{proof}
    For $S_4$ graphs $\omega(G)\leq\VC(V,\scHm(G))$ by Proposition~\ref{prop:vc_s4_lower}, which is a lower bound on the optimal number of mistakes \citep{littlestone1988learning}.
\end{proof}
}

\subsection{Agnostic Online Learning}\label{sec:agnostic_online}
Compared to the realizable setting studied so far, in the agnostic case we have no guarantee that there exists a hypothesis $H \in \scHm$ that makes no mistakes with respect to the true labels.
In other words, the true labels revealed by the environment during the online learning process can be determined according to any (fixed) arbitrary $f \in 2^V$.
Even though we now cannot show generally good mistake bounds as the ones in the realizable setting, we can at least provide guarantees on the performance with respect to the best fixed hypothesis, i.e., $H^* \in \argmin_{H \in \scHm(G)} M(H,f)$.
Indeed, if we let $M^* = M(H^*,f)$ be the minimum number of mistakes achieved by some hypothesis in $\scHm(G)$, we can design algorithms $A$ whose mistake bound $M(A,f)$ is a function of $M^*$.
We can show, in particular, that $M(A,f)-M^*$ can be upper bounded to provide guarantees for online learning in the agnostic case.\footnote{Note that $M(A,f)-M^*$, or better $\mathbb{E}[M(A,f)]-M^*$ if $A$ is randomized, takes the name of (expected) \emph{regret}.}

In order to achieve such guarantees, though, we can easily incur in large running times if we adopt one of the common strategies for agnostic online learning.
For example, adopting a learning-with-expert-advice algorithm $A$ such as Hedge \citep{cesa1997use}---also see, e.g., \citet[Chapter~2]{cesabianchi2006prediction}---where the experts are the hypotheses $\scHm(G)$ listed by \AlgoList, would provide an expected mistake bound of
\begin{equation}
    \mathbb{E}[M(A,f)] = M^* + \scO\left(\sqrt{T \log|\scHm(G)|}\right) = M^* + \scO\left(\sqrt{T\Bigl(\omega(G)+ \log\frac{m}{\omega(G)}\Bigr)}\right)
\end{equation}
for any $f \in 2^V$, without prior knowledge on $\omega(G)$ or $M^*$.
Moreover, one could provide an improved guarantee with a finer tuning for the learning rate \citep{ben2009agnostic} of $A$ to further show that
\begin{align}
    \mathbb{E}[M(A,f)] &\le M^* + \sqrt{2 M^* \ln|\scHm(G)|} + \ln|\scHm(G)| \\
    &= M^* + \scO\left(\sqrt{M^* \Bigl(\omega(G) + \log\frac{m}{\omega(G)}\Bigr)} + \omega(G) + \log\frac{m}{\omega(G)}\right) \enspace,
\end{align}
which can similarly be achieved by lifting prior knowledge on $M^*$ via a doubling trick.
However, approaches that implement learning with expert advice in a straightforward way would require to maintain weights over all $|\scHm(G)| \le \frac{4m 2^{\omega(G)}}{\omega(G)}+2$ hypotheses, thus incurring in $\scO(2^{\omega(G)} \poly(n))$ running time per round.

If we instead run the Winnow algorithm over the same task, it is possible to achieve similar rates (only loosing a factor of roughly $\omega(G)$) for the mistake bound while keeping computational efficiency even in the agnostic case.
As shown by \citet{blum1996online}, Winnow has the following mistake bound in the agnostic case.
\begin{proposition}[{\citet[Theorem~6]{blum1996online}}]
    For any sequence of examples consistent with any $f \in 2^{[n]}$ and any $r$-literal disjunction $h$ over $n$ variables, the number of mistakes made by Winnow is $\scO\bigl(r(M(h,f) + \log n)\bigr)$.
\end{proposition}
This immediately implies Theorem~\ref{thm:agnosticwinnow} for agnostic online learning with $\scHm(G)$.

\end{document}